%% file: main.tex
\documentclass{article}
\usepackage{etoolbox}

\providetoggle{workingversion}
\settoggle{workingversion}{false}
\providetoggle{cameraready}
\settoggle{cameraready}{false}
\providetoggle{draft}
\settoggle{draft}{false}
\providetoggle{preprint}
\settoggle{preprint}{true}

\providetoggle{iclr}
\providetoggle{arxiv}
\providetoggle{icml}
\providetoggle{neurips}
\settoggle{iclr}{false}
\settoggle{arxiv}{false}
\settoggle{icml}{true}
\settoggle{neurips}{false}

\newcommand{\ourtitle}{Excess Description Length of Learning Generalizable Predictors}

\newcommand{\authorlist}{Elizabeth Donoway, Hailey Joren, John Schulman, Fabien Roger, Jan Leike}
\newcommand{\authoretal}{Donoway et al.}

\iftoggle{iclr}{
\usepackage{ext/iclr/iclr2026_conference, times}
\iftoggle{cameraready}{
\lhead{Published as a conference paper at ICLR 2026}
}{
\lhead{Under review as a conference paper at ICLR 2026}
}
}{}

\iftoggle{arxiv}{
    \usepackage{ext/jmlr/jmlr2e_preprint}
    \usepackage{algorithm}
    \jmlrheading{1}{\the\year}{}{}{}{\authorlist}
    \ShortHeadings{\ourtitle}{\authoretal}
    \firstpageno{1}
}{}

\iftoggle{neurips}{
\usepackage[nonatbib]{ext/neurips/neurips_2025}
\title{\ourtitle{}}
\author{Author Name \\
Author1 Affil
}
}{}

\iftoggle{icml}{
\usepackage{natbib}
\iftoggle{preprint}{\usepackage[preprint]{ext/icml/icml2026}}{\usepackage{ext/icml/icml2026}}
}{}

\usepackage[utf8]{inputenc} 
\usepackage[T1]{fontenc}    
\usepackage{microtype}      
\usepackage{pifont}

\usepackage{environ,comment}
\usepackage{amsthm}
\usepackage{amsfonts,bbm,bm,courier}
\usepackage{amsmath,amsthm,amssymb,nicefrac}
\usepackage{graphicx,placeins,xcolor}
\usepackage{array,booktabs,tabularx,multirow,colortbl,hhline}
\usepackage{adjustbox}

\usepackage{enumitem}
\usepackage{wrapfig}
\usepackage{tikz}
\usepackage{standalone}
\usepackage{svg}
\usepackage{textcomp}
\usepackage{todonotes}
\usepackage{xfrac}
\usepackage{aliascnt}
\usepackage{pgfplots}
\usepgfplotslibrary{groupplots,fillbetween}
\usetikzlibrary{patterns,arrows.meta,calc}
\pgfplotsset{compat=newest}

\usepackage{hyperref}
\hypersetup{colorlinks=true, urlcolor=blue, linkcolor=blue, citecolor=blue, linkbordercolor=blue, pdfborderstyle={/S/U/W 1}}

\theoremstyle{plain}
\newtheorem{theorem}{Theorem}[section]

\newaliascnt{proposition}{theorem}
\newtheorem{proposition}[proposition]{Proposition}
\newaliascnt{corollary}{theorem}
\newtheorem{corollary}[corollary]{Corollary}
\theoremstyle{definition}
\newaliascnt{definition}{theorem}
\newtheorem{definition}[definition]{Definition}

\theoremstyle{remark}
\newaliascnt{remark}{theorem}
\newtheorem{remark}[remark]{Remark}

\newcommand{\addtexttilde}{\raisebox{0.5ex}{\texttildelow}}

\aliascntresetthe{definition}
\aliascntresetthe{proposition}
\aliascntresetthe{corollary}
\aliascntresetthe{remark}

\iftoggle{iclr}{
\title{\ourtitle{}}
\author{%
Author Name\\Affiliation\\\texttt{hName@ucsd.edu}\\
}
}

\begin{document}

\iftoggle{draft}{
    \lhead{\color{red} Draft}
}

\iftoggle{icml}{
\twocolumn[
    \icmltitle{\ourtitle{}}
    \begin{icmlauthorlist}
    \icmlauthor{Elizabeth Donoway}{ucb}
    \icmlauthor{Hailey Joren}{ant}
    \icmlauthor{Fabien Roger}{ant}
    \icmlauthor{Jan Leike}{ant}
    \end{icmlauthorlist}

    \icmlaffiliation{ucb}{Department of Physics, University of California, Berkeley, CA USA}
    \icmlaffiliation{ant}{Anthropic, San Francisco, CA USA}

    \icmlcorrespondingauthor{Elizabeth Donoway}{donoway@berkeley.edu}

    \icmlkeywords{Machine Learning, Information Theory, Learning Theory, Large Language Models}

    \vskip 0.3in
]
\printAffiliationsAndNotice{}
}{}

\iftoggle{neurips}{\maketitle}{}
\iftoggle{iclr}{\maketitle}{}

\iftoggle{arxiv}{
\title{\ourtitle{}}
\author{%
\name Elizabeth Donoway \email donoway@berkeley.edu\footnote{Correspondence to \texttt{donoway@berkeley.edu}} \\ \addr Anthropic \& University of California, Berkeley
\AND
\name Hailey Joren \email \\ \addr Anthropic 
\AND
\name John Schulman \email \\ \addr Thinking Machines
\AND
\name Ethan Perez \\ \addr Anthropic
\AND
\name Fabien Roger \\ \addr Anthropic
\AND
\name Jan Leike \\ \addr Anthropic
}
\maketitle
}

\begin{abstract}

Understanding whether fine-tuning elicits latent capabilities or teaches new ones is a fundamental question for language model evaluation and safety. We develop a formal information-theoretic framework for quantifying how much predictive structure fine-tuning extracts from the train dataset and writes into a model's parameters. Our central quantity, Excess Description Length (EDL), is defined via prequential coding and measures the gap between the bits required to encode training labels sequentially using an evolving model (trained online) and the residual encoding cost under the final trained model. We establish that EDL is non-negative in expectation, converges to surplus description length in the infinite-data limit, and provides bounds on expected generalization gain. Through a series of toy models, we clarify common confusions about information in learning: why random labels yield EDL near zero, how a single example can eliminate many bits of uncertainty about the underlying rule(s) that describe the data distribution, why structure learned on rare inputs contributes proportionally little to expected generalization, and how format learning creates early transients distinct from capability acquisition. This framework provides rigorous foundations for the empirical observation that capability elicitation and teaching exhibit qualitatively distinct scaling signatures.

\end{abstract}

\input{notation}

\section{Introduction}

Large language models acquire diverse capabilities during pretraining, many of which are not immediately apparent from their zero-shot behavior. Post-training interventions such as fine-tuning can dramatically improve performance on specific tasks, but the nature of this improvement varies fundamentally across settings. In some cases, fine-tuning appears to \textit{elicit} latent capabilities in the model, surfacing structure that was already present but not readily accessible. In other cases, fine-tuning must \textit{teach} genuinely new capabilities, encoding predictive structure that the model previously lacked.

This distinction is meaningful both practically and for safety. If a capability is latent, it may be elicited with minimal data and compute, potentially by actors with limited resources. If a capability must be taught, the information barrier is higher and more predictable. Yet standard metrics for evaluating fine-tuning---accuracy, loss curves, minimum description length (MDL), sample complexity---do not cleanly distinguish these regimes. Both elicitation and teaching can produce improved accuracy or loss with more data. Both can start from identical zero-shot performance, since a latent capability may be indistinguishable from an absent one when comparing only demonstrated behavior.

We seek a framework that provides a quantitative measure of the task-relevant information (in bits) stored in model parameters during training, with computable and well-defined operational semantics that connect to expected generalization through distributional arguments. Such a framework should explain why teaching and elicitation exhibit different scaling behaviors and provide foundations for empirical investigation of these phenomena.

While existing information-theoretic quantities like minimum description length (MDL) and sample complexity provide a powerful language for analyzing learning, they are often defined in the asymptotic, infinite-data limit. They excel at describing overall information requirements but say little about the pragmatic question of how far a small, finite amount of information can be stretched. It matters whether such an amount (in the form of high-quality examples or changes to model weights) is sufficient to unlock a broad capability, as this implies compute constraints may be the primary barrier to surfacing emergent behaviors.

Our approach adopts a prequential coding perspective. Imagine encoding the labels of a training dataset sequentially, where at each step the current model's predictive distribution determines the codelength for the next label, and the model is then updated on that example. The total codelength across the dataset is the prequential minimum description length (MDL). After training completes, the final model achieves some test loss on held-out data drawn from the same distribution. The difference between the prequential MDL and the codelength that would be required to encode the data using the final model, which we term the Excess Description Length (EDL), measures the predictive information that has been absorbed from the training data into the model's parameters.\footnote{Throughout, ``(generalizable) information absorbed'' refers to operational, algorithm-dependent predictive compression, not semantic or representation-level information.}

This paper makes three main contributions. First, we provide a formal definition of EDL with explicit dependence on the training algorithm, establishing it as a well-defined operational quantity and measure of information. Second, we prove basic properties of EDL: non-negativity in expectation, relationship to online learning regret, and convergence to surplus description length in the infinite-data limit. Third, we develop a series of toy models that resolve common confusions about information in learning and clarify when EDL signatures indicate elicitation versus teaching.

\iftoggle{workingversion}{\clearpage}{}
\section{Information-Theoretic Background and Definitions}
\label{sec::InformationTheoreticBackground}

Here, we briefly define and describe the information-theoretic quantities which form the foundation of our framework, establishing notation and reviewing the technical background necessary for our formal development of EDL in \autoref{sec::EDL}. We cover cross-entropy loss and its interpretation as codelength, the prequential coding scheme from online learning theory, and the minimum description length principle that motivates our approach. Detailed explanations and full derivations for all quantities, as well as visualizations of the information-theoretic properties associated with each learning regime can be found in \autoref{sec::InformationTheoryAppendix}.

\subsection{Setup and Notation}
\label{sec::Notation}

We consider supervised learning from a dataset $D = \{(x_i, y_i)\}_{i=1}^n$ where each example consists of an input $x_i \in \mathcal{X}$ and corresponding label $y_i \in \mathcal{Y}$. We assume examples are drawn independently from a true data distribution $\mathcal{D}$ over $\mathcal{X} \times \mathcal{Y}$. A model with parameters $\theta \in \Theta$ defines a conditional distribution $p_\theta(y \mid x)$ over labels given inputs.

The cross-entropy loss of model $\theta$ on example $(x, y)$ is
\begin{equation}
    \ell(\theta; x, y) = -\log p_\theta(y \mid x),
\end{equation}
where the logarithm is taken base $e$ (yielding nats) or base 2 (yielding bits). We will work primarily in nats and convert to bits by dividing by $\ln 2$ when reporting results. The expected loss under distribution $\mathcal{D}$ is
\begin{equation}
    L(\theta) = \mathbb{E}_{(x,y) \sim \mathcal{D}}[\ell(\theta; x, y)].
\end{equation}

A training algorithm $A$ specifies how parameters are updated given data. Formally, $A$ maps current parameters and a training example (or batch) to updated parameters: $\theta_{i} = A(\theta_{i-1}, (x_i, y_i))$. The algorithm $A$ encompasses the choice of optimizer, learning rate schedule, and all other hyperparameters that affect the sequence of parameter updates. Given initial parameters $\theta_0$ and dataset $D$, the algorithm produces a sequence of parameters $\theta_0, \theta_1, \ldots, \theta_n$ where $\theta_i$ denotes the parameters after training on the first $i$ examples.

Throughout, $\theta^*$ denotes the parameters after training completes (the final trained model parameters, considerate of any early stopping condition), which are not necessarily the globally optimal parameters. When we need to refer to optimal parameters, we write $\theta_\text{opt} = \mathrm{argmin}_\theta L(\theta)$.

\subsection{Cross-Entropy Loss as Compression}
\label{sec::CELoss}

The cross-entropy loss admits a fundamental interpretation in terms of information theory. Shannon's source coding theorem establishes that if we wish to communicate symbols drawn from a distribution $P$ using a code designed for distribution $Q$, the expected codelength per symbol is $H(P, Q) = -\mathbb{E}_{x \sim P}[\log Q(x)]$, the cross-entropy between $P$ and $Q$. This quantity is minimized when $Q = P$, in which case it equals the entropy $H(P)$.

For a single realization, the instantaneous codelength to encode outcome $y$ using a code based on distribution $Q$ is $-\log Q(y)$. When a model assigns probability $p_\theta(y \mid x)$ to the true label $y$ given input $x$, the cross-entropy loss $\ell(\theta; x, y) = -\log p_\theta(y \mid x)$ is precisely the number of nats (or bits, with appropriate base) required to encode $y$ using the model's predictive distribution as the coding scheme.

This interpretation transforms questions about learning into questions about compression. A model that predicts well, assigning high probability to correct labels, is equivalently a model that can compactly encode those labels. Improvements in prediction correspond directly to reductions in codelength.

\subsection{Prequential Coding}
\label{sec::PrequentialCoding}

Prequential (predictive sequential) coding provides a framework for encoding a sequence of observations without requiring a separate description of the model used for encoding. The fundamental principle, developed by \citet{Dawid1984prequentialstats} and refined in the context of minimum description length \cite{rissanen1978modeling}, is to encode each observation using the predictive distribution conditioned on all previous observations, then update the predictor before encoding the next observation.

Consider encoding a sequence of labels $y_1, \ldots, y_n$ given corresponding inputs $x_1, \ldots, x_n$. In prequential coding, we proceed as follows: encode $y_1$ using the initial model's prediction $p_{\theta_0}(y \mid x_1)$, incurring codelength $\ell(\theta_0; x_1, y_1)$; update parameters to $\theta_1 = A(\theta_0, (x_1, y_1))$; encode $y_2$ using the updated model's prediction $p_{\theta_1}(y \mid x_2)$, incurring codelength $\ell(\theta_1; x_2, y_2)$; and continue until all labels are encoded. The total codelength is
\begin{equation}
    \sum_{i=1}^{n} \ell(\theta_{i-1}; x_i, y_i).
\end{equation}

A notable property of this scheme is that the decoder can exactly replicate the encoder's model at each step, since both parties perform identical updates on exact copies of the same model. No separate transmission of model parameters is required—the model is implicitly communicated through the sequential encoding process. This provides an operational definition of the information content of a dataset relative to a learning algorithm: the number of bits required to transmit the labels using the model's evolving predictions.

\subsection{Minimum Description Length}
\label{sec::MDL}

The Minimum Description Length (MDL) principle, introduced by \citet{rissanen1978modeling}, operationalizes the parsimony principle for model selection, formalizing the intuition that the best model is one that most compresses the data. In its two-part form, the description length of data $D$ using model $M$ from a hypothesis class $\mathcal{M}$ is
\begin{equation}
    L(D; M) = L(M) + L(D \mid M),
\end{equation}
where $L(M)$ is the codelength required to describe the model and $L(D \mid M)$ is the codelength of the data given the model. The MDL principle selects the model minimizing total description length, \(L(D; M)\).

Prequential coding provides an alternative formulation that avoids explicitly coding the model. The prequential codelength depends on the learning algorithm but not on any model complexity penalty, since the model is never directly transmitted. This makes prequential MDL particularly suitable for analyzing neural network fine-tuning, where the relationship between parameter count and effective complexity is unclear.

\subsection{Surplus Description Length}
\label{sec::SDL}

Surplus description length (SDL), introduced by \citet{SDLwhitney2021evaluatingrepresentationscomplexitylearning}, quantifies the excess codelength incurred during learning relative to the optimal achievable loss. For a dataset $D$ and algorithm $A$ producing final parameters $\theta^*$, the SDL is defined as
\begin{equation}
    \text{SDL}(D; A) = \sum_{i=1}^n \ell(\theta_{i-1}; x_i, y_i) - n \cdot L^*,
\end{equation}
where $L^* = \inf_\theta L(\theta)$ is the optimal expected loss achievable within the model class.

SDL captures the information cost of learning: additional bits spent during training beyond what the optimal model would require. It provides a measure of how quickly or efficiently an algorithm converges to optimal predictions. However, SDL requires access to an (effectively) infinite distribution of data to sample training examples from, which may be impractical and does not reflect data constraints common in fine-tuning settings. SDL also requires knowledge of $L^*$, which may be unknown or uncomputable in practice. EDL addresses these points by replacing the optimal loss with the achieved test loss, yielding a fully computable measure that remains valid for any dataset size.

\iftoggle{workingversion}{\clearpage}{}
\section{Formal Definitions}
\label{sec::ExcessDescriptionLength}

We now present formal definitions of prequential MDL and excess description length (EDL). These definitions make explicit the dependence on the training algorithm and provide clear operational implementations: both quantities can be computed directly from training logs and finite data without access to any oracle or asymptotic limit.

\subsection{Prequential Minimum Description Length}
\label{sec::PrequentialMDL}

\begin{definition}[Prequential MDL]
    \label{def::PrequentialMDL}
    Let $D = \{(x_i, y_i)\}_{i=1}^n$ be a training dataset, $\theta_0$ initial model parameters, and $A$ a training algorithm. The prequential minimum description length is
    \begin{equation}    
        \mathrm{MDL}(D; \theta_0, A) = \sum_{i=1}^{n} \ell(\theta_{i-1}; x_i, y_i),
    \end{equation}
    where $\theta_i = A(\theta_{i-1}, (x_i, y_i))$ for $i \geq 1$.
\end{definition}

The prequential MDL measures the total codelength required to encode the training labels using the evolving model as the coding distribution. Each term $\ell(\theta_{i-1}; x_i, y_i)$ represents the bits needed to encode label $y_i$ using the model's prediction \textit{before} updating on example $i$. The sum accumulates these costs across the entire dataset.

In practice, training proceeds by batches rather than single examples. The definition extends naturally: for batch $B_j$ processed at step $j$, we accumulate $\sum_{(x,y) \in B_j} \ell(\theta_{j-1}; x, y)$ before the update, where $\theta_{j-1}$ denotes parameters before processing batch $B_j$. The order of summation within a batch does not affect the total, and the prequential interpretation remains valid since all examples in the batch are encoded using the pre-update parameters.

The prequential MDL depends on the order in which examples are presented. For randomly shuffled data, this dependence introduces variance but does not change the expected value under the data distribution. When comparing MDL across experiments, the same ordering (or averaging over orderings) should be used.

\subsection{Excess Description Length}
\label{sec::EDL}

Training typically continues beyond a single pass through the data, with multiple epochs refining the model's parameters. Let $\theta^*$ denote the final parameters after training completes, whether by convergence, early stopping, or a fixed compute budget. The test loss of the final model is
$$L_{\text{test}}(\theta^*) = \mathbb{E}_{(x,y) \sim \mathcal{D}_{\text{test}}}[\ell(\theta^*; x, y)],$$
estimated in practice by averaging over a held-out test set drawn from the same distribution as the training data.

\begin{definition}[Excess Description Length]
    \label{def::EDL}
    The excess description length is
    \begin{equation}
        \mathrm{EDL}(D; \theta_0, A) = \mathrm{MDL}(D; \theta_0, A) - n \cdot L_{\mathrm{test}}(\theta^*).
    \end{equation}
\end{definition}

The EDL measures the gap between two quantities: the bits required to encode training labels using the evolving model during first exposure (MDL), and the bits that would be required to encode $n$ new samples from the same distribution using the final model ($n \cdot L_{\text{test}}(\theta^*)$). This gap represents the predictive information that has been absorbed into the model's parameters—structure extracted from the training data that now enables more efficient encoding of test data.

To see this interpretation more clearly, consider the following decomposition. The MDL can be written as
\begin{equation}
    \text{MDL} = \sum_{i=1}^n \left[\ell(\theta_{i-1}; x_i, y_i) - L_{\text{test}}(\theta^*)\right] + n \cdot L_{\text{test}}(\theta^*).
\end{equation}
The first term sums the ``excess'' loss at each step---how much worse the current model is than the final model---while the second term is the residual codelength using the final model. Rearranging, we have
\begin{equation}
    \text{EDL} = \sum_{i=1}^n \left[\ell(\theta_{i-1}; x_i, y_i) - L_{\text{test}}(\theta^*)\right],
\end{equation}
which makes explicit that EDL accumulates the per-example excess over the training trajectory.

\subsection{Interpretation as a Communication Protocol}
\label{sec::CommsInterp}

The prequential coding framework admits a compelling interpretation in terms of communication between two parties. Suppose Alice has trained a model and wishes to communicate the results of her training to Bob. They start with identical copies of the initial model $\theta_0$ and have agreed on the training algorithm $A$, including all hyperparameters necessary to replicate the model and training process. Alice has the full training dataset $D$; Bob has only the inputs $\{x_i\}_{i=1}^n$.

Rather than transmitting the trained parameters directly, which may require billions of bits for large models, Alice can transmit the training labels, from which Bob can reconstruct an identical trained model. The protocol proceeds as follows. For each example $i$, Alice encodes $y_i$ using their shared current model's predictive distribution $p_{\theta_{i-1}}(y \mid x_i)$ and transmits this encoding to Bob. Bob decodes $y_i$ using the same distribution. Both parties then update their parameters identically: $\theta_i = A(\theta_{i-1}, (x_i, y_i))$. After processing all examples, both have identical final parameters $\theta^*$.

The total communication cost is precisely the prequential MDL. This cost reflects the information in the training labels as measured by the model's evolving ability to predict them. Labels that the model already predicts well require few bits; labels that surprise the model require many bits.

The test loss $L_{\text{test}}(\theta^*)$ measures how many bits per example the final model would still require to encode new samples from the distribution. If Bob, after receiving Alice's transmission, encounters new examples from $\mathcal{D}$, he expects to need $L_{\text{test}}(\theta^*)$ bits per example to encode their labels. The EDL thus measures the difference between what Alice transmitted and what remains unexplained by the final model---the information Alice sent that has been ``absorbed'' (about $\mathcal{D}$) into the model rather than remaining as residual encoding cost.

\subsection{Normalizations}
\label{sec::Normalizations}

For comparison across datasets of different sizes or models with different parameter counts, we employ several normalizations of EDL.

The per-example EDL is $\text{EDL}_{\text{ex}} = \text{EDL} / n$, measuring average absorbed information per training example. This normalization is useful for comparing learning efficiency across dataset sizes.

The per-token EDL is $\text{EDL}_{\text{tok}} = \text{EDL} / D$, where $D$ is the total number of tokens in the dataset scored in computing MDL (typically, the label tokens across all examples). For language modeling tasks where examples vary in length, this normalization accounts for the total volume of prediction.

The per-parameter EDL is $\text{EDL}_{\text{par}} = \text{EDL} / P$, where $P$ is the number of trainable parameters. This normalization is particularly relevant for parameter-efficient fine-tuning methods like LoRA, where only a subset of parameters are updated. It measures the information (bits) compressed per trainable parameter, providing insight into whether the parameter budget constrains the information that can be absorbed.

\subsection{Relationship to Surplus Description Length}
\label{sec::SDLEDLRelationship}

EDL is closely related to the surplus description length of \citet{SDLwhitney2021evaluatingrepresentationscomplexitylearning}, differing primarily in the reference point used. SDL subtracts $n \cdot L^*$, where $L^*$ is the optimal achievable loss; EDL subtracts $n \cdot L_{\text{test}}(\theta^*)$, the achieved test loss.

When the training algorithm successfully finds near-optimal parameters---that is, when $L_{\text{test}}(\theta^*) \approx L^*$---the two quantities coincide. In general, EDL satisfies
\begin{equation}
    \text{EDL} = \text{SDL} + n \cdot (L^* - L_{\text{test}}(\theta^*)).
\end{equation}
The additional term $n \cdot (L^* - L_{\text{test}}(\theta^*))$ is non-positive, reflecting that achieved test loss is at least as large as optimal loss. Thus $\text{EDL} \leq \text{SDL}$, with equality when the learner achieves the optimum.

The practical advantage of EDL over SDL is computability. Computing SDL requires knowing $L^*$, which may be unavailable, especially for complex distributions or model classes where the optimal parameters are unknown. EDL requires only quantities that are directly measurable: the accumulated training loss during the first epoch and the test loss of the final model.

\iftoggle{workingversion}{\clearpage}{}
\section{Properties and Bounds}
\label{sec::EDLPropertiesBounds}

This section establishes fundamental properties of EDL, including non-negativity, relationships to online learning regret, and asymptotic behavior. We demonstrate that EDL satisfies the properties of operationally derived information-theoretic measures and is a valid description length. These results formalize the intuition that EDL measures absorbed generalizable information and provide foundations for interpreting empirical measurements.

\subsection{Non-Negativity}
\label{sec::NonNegativity}

A basic verification for any measure of ``information absorbed'' is that it should be non-negative: training should not destroy predictive information, on average. We establish this property under standard assumptions.

\begin{definition}[Population-Monotonic Algorithm]
\label{def:populationmonotonic}
    A training algorithm $A$ is population-monotonic with respect to distribution $\mathcal{D}$ if, for all steps $j$ in the training trajectory,
    \begin{equation}
        \mathbb{E}[L(\theta_j) \mid \theta_{j-1}] \leq L(\theta_{j-1}),
    \end{equation}
    where $L(\theta) = \mathbb{E}_{(x,y) \sim \mathcal{D}}[\ell(\theta; x, y)]$ is the population loss. That is, each update does not increase the expected population loss in expectation over the randomness in the update (e.g., batch sampling).

    This condition is satisfied by:
    \begin{itemize}
        \item Gradient descent on convex losses with appropriate learning rate
        \item SGD with sufficiently small learning rate on smooth losses
        \item Any algorithm where early stopping prevents overfitting
    \end{itemize}

    This condition may be violated by:
    \begin{itemize}
        \item Large learning rates that cause oscillation
        \item Training well past the optimal early stopping point on finite data
        \item Algorithms that memorize training data at the expense of generalization
    \end{itemize}
\end{definition}

\begin{theorem}[Non-negativity in Expectation]
\label{thm:nonneg}
    Let $D$ be drawn i.i.d. from distribution $\mathcal{D}$, let the test loss be evaluated on an independent sample $D_\text{test}$ from $\mathcal{D}$, and let $A$ be a population-monotonic algorithm (\autoref{def:populationmonotonic}). Then
    \begin{equation}
        \mathbb{E}_{D, D_{\mathrm{test}}}[\mathrm{EDL}(D; \theta_0, A)] \geq 0.
    \end{equation}
\end{theorem}

\begin{remark}
    For algorithms that are not population-monotonic, EDL may be negative
    in expectation. Negative expected EDL indicates that the training procedure has degraded generalization---the model has ``unlearned'' predictive structure. This can occur when training far past the early stopping point or with pathological hyperparameters. In practice, we observe $\mathbb{E}[\text{EDL}] \geq 0$ for well-tuned training procedures (see companion empirical paper).
\end{remark}

\begin{proof}
    The expected test loss of the final model cannot exceed the expected test loss of the initial model, since training on data from $\mathcal{D}$ should not systematically worsen predictions on $\mathcal{D}$ for a population-monotonic algorithm. More formally, by the online-to-batch conversion, the expected loss of the final model is at most the average expected loss across the training trajectory.

    Write the expected MDL as
    \begin{equation}
        \mathbb{E}[\text{MDL}] = \sum_{i=1}^n \mathbb{E}[\ell(\theta_{i-1}; x_i, y_i)] = \sum_{i=1}^n L(\theta_{i-1}),
    \end{equation}
    where the second equality uses the i.i.d. assumption: conditioning on $\theta_{i-1}$ (which depends on $x_1, y_1, \ldots, x_{i-1}, y_{i-1}$), the expectation of $\ell(\theta_{i-1}; x_i, y_i)$ over the fresh draw $(x_i, y_i) \sim \mathcal{D}$ equals $L(\theta_{i-1})$.
    
    For the test loss, $\mathbb{E}[L_{\text{test}}(\theta^*)] = L(\theta^*)$ by independence of test data.
    
    The expected EDL is therefore
    \begin{equation}
        \mathbb{E}[\text{EDL}] = \sum_{i=1}^n L(\theta_{i-1}) - n \cdot L(\theta^*) = \sum_{i=1}^n \left[L(\theta_{i-1}) - L(\theta^*\right)].
    \end{equation}
    
    Each term $L(\theta_{i-1}) - L(\theta^*)$ is the expected excess loss of the intermediate model relative to the final model. Since $\theta^*$ is obtained by training on the full dataset (or beyond), and training generally improves expected loss, we have $L(\theta_{i-1}) \geq L(\theta^*)$ for all $i$. Thus each term is non-negative, and so is their sum.
\end{proof}

For finite test sets, the estimated test loss has variance, and individual realizations of EDL may be slightly negative. This sampling noise vanishes as test set size increases.

\subsection{Relationship to Regret}
\label{sec::Regret}

Online learning theory provides bounds on cumulative loss through the notion of regret. The regret of an algorithm relative to a fixed comparator $\theta$ is
\begin{equation}
    R_n(\theta) = \sum_{i=1}^n \ell(\theta_{i-1}; x_i, y_i) - \sum_{i=1}^n \ell(\theta; x_i, y_i).
\end{equation}
This measures how much worse the online algorithm performs compared to having used $\theta$, the best model from the hypothesis class for each example, throughout.

EDL differs from regret in that regret quantifies the cumulative excess \textit{train set loss} relative to that of the fixed, globally optimal predictor of the specific sequence of train data, agnostic to arbitrary redundancy in examples, whereas EDL quantifies the excess \textit{minimum expected codelength} required to encode all unique examples in the train set using the online-trained model rather than a fixed predictor with constant generalization error.

Regret measures the additional codelength incurred through online learning because the best model \textit{a posteriori} for encoding each example was not used. EDL quantifies the additional codelength incurred through learning because the initial model did not generalize as well as could have. Regret quantifies excess codelength relative to an optimal, fixed predictor of the train set data, whereas EDL quantifies excess codelength relative to a fixed predictor which generalizes over the test distribution with error equal to the test loss. EDL, regret, and SDL all converge to the same value in the asymptotic limit.

\begin{theorem}[MDL and Regret]
    \label{thm:MDLRegret}
    For any fixed $\theta \in \Theta$,
    \begin{equation}
        \mathrm{MDL}(D; \theta_0, A) = \sum_{i=1}^n \ell(\theta; x_i, y_i) + R_n(\theta).
    \end{equation}
\end{theorem}

This decomposition expresses MDL as the loss of a fixed model plus the regret incurred by the learning algorithm. Taking $\theta = \theta^*$ (the final trained parameters) and noting that $\sum_{i=1}^n \ell(\theta^*; x_i, y_i) \approx n \cdot L_{\text{train}}(\theta^*)$ where $L_{\text{train}}$ is the training loss, we obtain
\begin{equation}
    \text{MDL} \approx n \cdot L_{\text{train}}(\theta^*) + R_n(\theta^*).
\end{equation}

Standard online learning algorithms achieve regret bounds of $O(\sqrt{n})$ for convex losses, implying that $\text{MDL}/n \to L_{\text{train}}(\theta^*)$ as $n \to \infty$. The prequential codelength per example converges to the per-example training loss of the final model.

\begin{corollary}
    If the algorithm achieves sublinear regret $R_n(\theta^*) = o(n)$, then
    \begin{equation}
        \frac{\mathrm{EDL}}{n} \to L_{\mathrm{train}}(\theta^*) - L_{\mathrm{test}}(\theta^*)
    \end{equation}
    as $n \to \infty$. When the model generalizes perfectly ($L_{\mathrm{train}} = L_{\mathrm{test}}$), the per-example EDL vanishes asymptotically.
\end{corollary}

This corollary connects EDL to the generalization gap. Non-zero per-example EDL in the asymptotic regime indicates a gap between training and test performance due to overfitting or distribution shift. 

As the asymptotic regime is characterized by i.i.d. sampling $n \rightarrow \infty$ examples from an (effectively) infinite distribution, non-zero per-example EDL indicates redundancy in train or test examples, spurious correlations within the set of training examples (which should vanish in expectation as $n \rightarrow \infty$), and/or a distribution shift between the train and test distributions. Large, negative per-example EDL in the asymptotic regime---while not a valid codelength---does have practical value, revealing that the model has absorbed information specific to the training set that does not transfer to test data.

\subsection{Asymptotic Convergence to SDL}
\label{sec::AsymptoticConvergenceSDL}

As dataset size grows, EDL converges to surplus description length under conditions ensuring that the trained model approaches optimality.

\begin{theorem}[Convergence to SDL]
    \label{thm:ConvToSDL}
    Suppose the learning algorithm is consistent for the model class $\Theta$ under distribution $\mathcal{D}$: $L(\theta^*) \to L^*$ almost surely as $n \to \infty$, where $L^* = \inf_{\theta \in \Theta} L(\theta)$. Then
    \begin{equation}
        \frac{\mathrm{EDL} - \mathrm{SDL}}{n} = n \rightarrow 0 \quad \text{as} \quad n \rightarrow \infty.
    \end{equation}
    Specifically,
    \begin{equation}
        \frac{\mathrm{EDL} - \mathrm{SDL}}{n} = n \cdot \left(L^* - L_\mathrm{test}(\theta^*)\right),
    \end{equation}
    and under consistency, this difference grows sublinearly in $n$.
\end{theorem}

\begin{remark}
    The rate of convergence depends on the model class:
    \begin{itemize}
        \item For finite hypothesis classes: $O\left( \sqrt{\frac{\log |\Theta|}{n}} \right)$
        \item For parametric models with bounded complexity: $O\left( \frac{1}{\sqrt{n}} \right)$
        \item For overparameterized neural networks: rates depend on implicit regularization and may not follow classical scaling
    \end{itemize}

    We do not claim rates for neural networks. The theorem establishes qualitative convergence under the consistency assumption, which can be verified emprically for specific architectures and training procedures.
\end{remark}

Under consistency and given suitable rates of convergence, the gap between achieved and optimal loss vanishes, and EDL coincides with SDL. This provides theoretical grounding for using EDL as a practical, finite-data analog to the information-theoretically motivated SDL.

\subsection{Bound on Expected Generalization Error Improvement}
\label{sec::GeneralizationErrorBounds}

EDL provides an upper bound on how much the expected generalization error (loss) improves through training.

\begin{theorem}[Generalization Bound]
\label{thm:genbound}
    For a population-monotonic algorithm, the expected improvement in population loss decomposes as:
    \begin{equation}
        L(\theta_0) - \mathbb{E}\left[L(\theta^*)\right] = \left(L(\theta_0) - \bar{L}\right) + \frac{\mathbb{E}[\mathrm{EDL}]}{n}
    \end{equation}
    where $\bar{L} = \frac{1}{n}\sum_{i=1}^n \mathbb{E}\left[L(\theta_{i-1})\right]$ is the average expected population loss across the training trajectory.
\end{theorem}

Equivalently, expected EDL satisfies:
\begin{equation}
    \mathbb{E}[\text{EDL}] = n \cdot \left(\bar{L} - \mathbb{E}\left[L(\theta_{i-1})\right]\right).
\end{equation}

The total improvement $L(\theta_0) - \mathbb{E}\left[L(\theta^*)\right]$ can be decomposed into two parts:
\begin{enumerate}
    \item $L(\theta_0) - \bar{L}$: improvement during the training trajectory (always non-negative for population-monotonic algorithms)
    \item $\mathbb{E}[\text{EDL}]/n$: per-example absorbed generalizable information, measuring how much better the final model performs on the population distribution compared to the trajectory average
\end{enumerate}

\begin{proof}[Proof sketch (see \autoref{Sec::AppendixGeneralizationBoundProof} for full proof)]
    From the proof of \autoref{thm:nonneg},
    \begin{equation}
        \mathbb{E}[\text{EDL}] = \sum_{i=1}^n \left[L(\theta_{i-1}) - L(\theta^*)\right] = n \cdot \bar{L} - n \cdot L(\theta^*)
    \end{equation}
    Rearranging,
    \begin{equation}
        L(\theta^*) = \bar{L} - \frac{\mathbb{E}[\text{EDL}]}{n}.
    \end{equation}
    The result follows from $L(\theta_0) - L(\theta^*) = \left(L(\theta_0) - \bar{L}\right) + \left(\bar{L} - L(\theta^*)\right)$.
\end{proof}

This bound shows that \(\text{EDL}/n\) upper bounds the improvement attributable to learning during training (the $\bar{L} - L(\theta^*)$ term), while the remaining improvement $(L(\theta_0) - \bar{L})$ reflects loss reduction during early training before substantial learning occurs.

\subsection{Dependence on Learning Algorithm}
\label{sec::LearningAlgorithmDependence}

Computed prequentially, EDL inherently depends on the training algorithm and is not an intrinsic property of the model class and dataset alone.

\begin{proposition}[Algorithm Dependence]
\label{prop:algdependence}
    For fixed $D$ and $\theta_0$, different algorithms $A$ and $A'$ can yield different EDL values: $\mathrm{EDL}(D; \theta_0, A) \neq \mathrm{EDL}(D; \theta_0, A')$, in general.
\end{proposition}

This dependence arises through both MDL (determined by the loss trajectory) and the final test loss (determined by where training terminates). A more aggressive learning rate might achieve lower test loss but incur higher MDL due to initially worse predictions; a conservative learning rate might have lower MDL but higher test loss.

This algorithm dependence is not a defect but a feature: different training procedures genuinely differ in how efficiently they extract information from data. When comparing EDL across settings, the algorithm should be held fixed or explicitly varied as an experimental factor.

\iftoggle{workingversion}{\clearpage}{}
\section{Toy Models}
\label{sec::ToyModels}

To build intuition for EDL and clarify common confusions about information in learning, we present a series of toy models. Each illustrates a specific phenomenon and provides insight into when EDL signatures indicate elicitation versus teaching.

The following toy models are intended as conceptual illustrations rather than rigorous characterizations of neural network learning. Each model captures one phenomenon in isolation, whereas real learning involves interactions among multiple effects. The functional forms (e.g., linear learning curves, coupon collector dynamics) are chosen for analytical tractability and should be viewed as approximations. The qualitative predictions that teaching and elicitation have different EDL signatures are validated empirically in a companion paper.

\subsection{Random Labels: EDL Equals Zero When Nothing Generalizes}
\label{sec::RandomLabels}

\begin{figure}[ht]
    \centering
        \includegraphics[width=1.0\columnwidth]{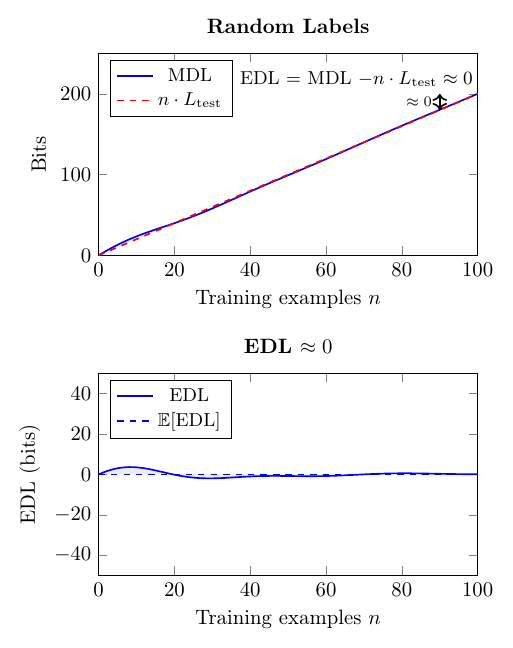}
    \caption{When labels are random, test loss $L_\text{test}$ remains constant since no generalizable pattern can be learned. MDL and the residual codelength $n \cdot L_\text{test}$ both increase linearly (in expectation), proportional to the number of training examples $n$. Since no learnable structure exists in the data, $\text{EDL} = \text{MDL} - n \cdot L_\text{test}\approx 0$, and the amount of generalizable information absorbed is negligible.}
    \label{fig:RandomLabelsDiagram}
\end{figure}

Our first toy model establishes a crucial baseline: when training labels carry no predictive information about test labels, EDL should be zero. This validates that EDL measures learned structure of the data distribution rather than mere computational effort or overfitting.

Consider a setting where labels are generated independently of inputs and independently between training and test sets. Specifically, let $y \sim \text{Uniform}(\mathcal{Y})$ independently for each example, with $|\mathcal{Y}| = k$. The training labels $\{y_i\}_{i=1}^n$ and test labels are drawn from this same marginal distribution but with no shared structure—knowing the training labels provides no information about test labels.

In this setting, the optimal predictor is the uniform distribution regardless of training data, achieving loss $\log k$ per example. The initial model, if well-calibrated, begins near this optimum. Training on random labels may cause the model to memorize the training set, but this memorization does not reduce test loss.

The prequential MDL is determined by the model's predictions during training. If the model maintains predictions near uniform (the rational response to recognizing random labels), MDL is approximately $n \log k$. If the model instead memorizes training labels, later training examples may have higher loss before being memorized, but the eventual test loss remains at $\log k$ since memorization does not generalize.

In either case, $\text{EDL} = \text{MDL} - n \cdot L_{\text{test}} \approx n \log k - n \log k = 0$. The model cannot compress the test labels any better than chance, so no predictive information has been absorbed from training.

\begin{proposition}
\label{prop::randomlabels}
    For i.i.d. random labels independent between train and test, $\mathbb{E}[\mathrm{EDL}] = 0$.
\end{proposition}

This result provides a simple and practical validation for empirical measurements: if EDL is substantially positive, there must be learnable structure in the data. Random labels yield EDL near zero regardless of the training effort or computational resources expended.

\subsection{Hypothesis Collapse: One Label Can Eliminate Many Bits of Uncertainty}
\label{sec::HypothesisCollapse}

A common confusion about information in learning is the belief that a single binary label can provide at most one bit of information. This is incorrect when the label resolves uncertainty about which rule governs the data.

Consider a learner with a hypothesis class $\mathcal{H} = \{h_1, \ldots, h_m\}$ representing possible input-output relationships for data distribution $\mathcal{D}: \mathcal{X} \times \mathcal{Y}$. The learner begins with a uniform prior over hypotheses. A single example $(x, y) \in \mathcal{X} \times \mathcal{Y}$ is observed, and suppose this example is \textit{diagnostic}: only one hypothesis $h_j$ is consistent with the observed input-output pair, while all others predict different labels for input $x$.

Before observing the example, the learner's uncertainty about which hypothesis is correct is $\log m$ bits. After observing $(x, y)$, this uncertainty drops to zero, as the learner now knows $h_j$ is the true hypothesis. The single example has provided $\log m$ bits of information about the underlying rule, even though the label $y$ itself has at most $\log |\mathcal{Y}|$ bits of entropy.

This phenomenon explains how EDL can be large even when training on few examples. If those examples are highly diagnostic, meaning they eliminate many hypotheses about how to perform a task, each example contributes information proportional to the logarithm of the number of hypotheses eliminated, as opposed to the naive label entropy.

\begin{proposition}[Hypothesis Collapse]
\label{prop:hypothesiscollapse}
    Consider a Bayesian learner with hypothesis class $\mathcal{H} = \{h_1, \ldots, h_m\}$ and uniform prior. Suppose:
    \begin{enumerate}[label=\roman*]
        \item The example $(x, y)$ is diagnostic, meaning that exactly one hypothesis $h_j$ satisfies $h_j(x)=y$.
        \item The hypothesis are maximally distinguishing on $x$: for each label $y' \in \mathcal{Y}$, exactly $m/|\mathcal{Y}|$ hypotheses predict $y'$ for input x.
    \end{enumerate}
    Then, a single example contributes exactly $\log |\mathcal{Y}|$ bits to EDL, and the generalization improvement is $\log m$ bits, the entropy of the full hypothesis space.

    More generally, without assumption (ii), a single example can contribute at most $\mathrm{min}\left(\log |\mathcal{Y}|, \log m\right)$ bits to EDL, where the codelength contribution is determined by the label entropy under the predictive distribution, while the generalization improvement is determined by the hypothesis entropy reduction.
\end{proposition}

Note that since EDL measures the amount of information \textit{present in the observed train data} that the model has gained about the test distribution, the model can gain at most as much generalizable information as the train data contains (\textit{i.e.}, EDL = MDL), even if a single example collapses the entire hypothesis space, reducing its uncertainty to zero.

In the context of elicitation, this toy model suggests why minimal fine-tuning can dramatically improve performance. If a pretrained model ``contains'' a capability in the sense that the correct hypothesis is among a small set,\footnote{For example, correct responses to the question, ``What is 2 + 2?'' may be ``4,'' ``A simple addition problem,'' or ``$2\times2$.'' Though the correct hypothesis will depend on the specifics of the task, the hypothesis space is limited, and a single example may suffice to remove the uncertainty about which to choose.} one or few examples may suffice to identify which hypothesis to use. The model already has the capability; the examples simply disambiguate which of its latent competencies to apply.

\subsection{Disjoint Subdistributions: Expected Generalization Gain Scales with Coverage}
\label{sec::DisjointSubdistributions}

Learning a rule that applies only to a subset of the input distribution contributes proportionally less to expected generalization. This toy model formalizes the intuition that rare patterns matter less for overall performance.

Let the data distribution $\mathcal{D}$ be a mixture of $K$ disjoint subdistributions:
\begin{equation}
    \mathcal{D} = \sum_{j=1}^K \pi_j \mathcal{D}_j,
\end{equation}
where each $\mathcal{D}_j$ has support $\mathcal{X}_j$ with $\mathcal{X}_j \cap \mathcal{X}_{j'} = \emptyset$ for $j \neq j'$, and $\sum_j \pi_j = 1$. Suppose that within each subdistribution, there is a ``rule'' that, once learned, reduces the per-example loss by 1 bit (or some fixed amount $\Delta$).

If the training data covers only subdistribution $j$ with mixture weight $\pi_j$, the model can learn the corresponding rule and reduce loss on inputs from $\mathcal{X}_j$. However, test data is drawn from the full mixture $\mathcal{D}$. The expected improvement in test loss is only $\pi_j \cdot \Delta$, the loss reduction on $\mathcal{X}_j$ weighted by the probability of encountering inputs from subdistribution \(j\).

\begin{proposition}
\label{prop:SubdistributionWeightEDL}
    Learning a rule that reduces loss by $\Delta$ bits on a subdistribution with mixture weight $\pi$ contributes approximately $\pi \cdot \Delta$ bits per example to expected generalization improvement.
\end{proposition}

This model clarifies the relationship between total information absorbed (computed on training data), expected generalization (computed over the test distribution), and generalizable information absorbed (EDL, computed on training and test data). A training set concentrated on a rare subdistribution may yield high total information absorbed if the model learns to predict that subdistribution well, but the expected generalization gain over the entire test distribution, EDL, is small because the learned structure applies to only a small fraction of test inputs.

In practice, this phenomenon affects the interpretation of EDL signatures. Low EDL does not necessarily indicate low practical improvement on all subsets of the test distribution if the training data overrepresents rare cases. Conversely, modest EDL concentrated on high-probability regions may translate to substantial practical gains if learning is rapid (such as is the case for elicitation).

\subsection{Coupon Collector Dynamics: Phase Transitions in Learning}
\label{sec::CouponCollector}

When a task requires exposure to multiple distinct concepts or pieces of knowledge before generalization is possible, marginal contributions to EDL (per train example or label token) can exhibit non-monotonic scaling with dataset size. This toy model illustrates how coverage requirements can result in phase transitions.

Consider a task with $n$ distinct concepts, where the model must see at least one example of each concept to generalize, which occurs at time $T$ (when all concepts have been encountered). This is analogous to the coupon collector's problem: here, examples are drawn uniformly from concepts (coupons), and the model collects concepts until all $n$ are represented.

In the coupon collector's problem, the expected number of draws to collect all $n$ coupons is $\mathbb{E}[T]=n \cdot H_n \approx n \log n$, where $H_n$ is the $n$-th harmonic number. Before achieving full coverage, the model cannot reliably predict examples from unseen concepts, and test loss remains high on those concepts.

The EDL trajectory shows three phases. Initially, when few concepts have been seen, MDL accumulates rapidly (each new example is surprising) but test loss remains high (many concepts are still unknown), yielding moderate EDL. As training progresses toward the coverage threshold, more concepts are learned, and both MDL accumulation slows (fewer novel concepts) and test loss drops (more concepts become predictable), causing EDL to increase. Beyond full coverage, additional examples provide diminishing information because they are redundant with known concepts, and EDL saturates or grows slowly.

\begin{proposition}
\label{prop:couponcollector}
    For a task requiring coverage of $n$ concepts, EDL shows an accelerating phase as $k$, the number of examples trained on so far, approaches $n \log n$, corresponding to the transition from partial to full coverage.
\end{proposition}

This model is relevant for teaching scenarios where the capability requires learning multiple components. The EDL signature of teaching---a phase where per-example EDL (significantly) increases---reflects the coverage dynamics. Until sufficient concepts are observed, the model cannot fully generalize, and additional examples provide increasing marginal value.

\begin{figure}
    \centering
        \includegraphics[width=1.0\columnwidth]{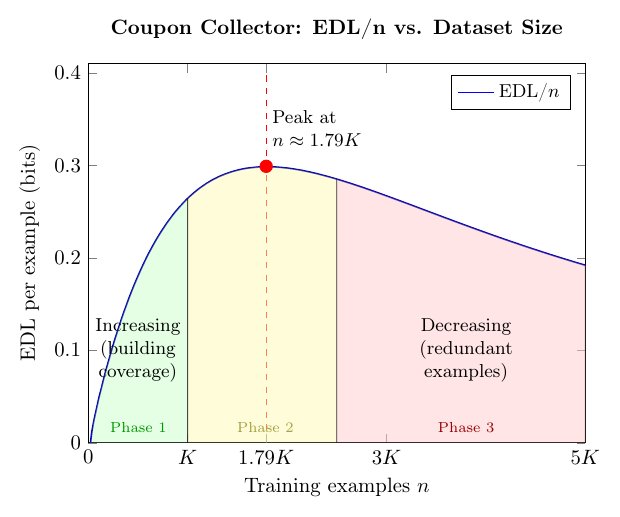}
    \caption{In the low-coverage regime (Phase 1, green), additional examples provide opportunities to learn generalizable patterns from relatively few known concepts. The rate of information absorption peaks as full coverage of concepts is approached (Phase 2, yellow). Beyond the coverage threshold (peak), learning enters the full-coverage regime (Phase 3, red), as additional examples yield diminishing predictive information that can be absorbed.}
    \label{fig:CouponCollectorDiagram}
\end{figure}

\subsection{Format Learning: Early Transients Versus Capability Acquisition}
\label{sec::FormatLearning}

Many tasks have two components: format (how to structure outputs) and capability (what to predict). Learning proceeds differently for these components, creating artifacts in early training that must be distinguished from capability acquisition.

Consider a task where the model must output answers in a specific format (e.g., ``The answer is X'' rather than just ``X''). Format is low-entropy and quickly learned, whereas capability is high-entropy and requires more data. A model fine-tuned on such a task shows rapid initial loss reduction as format is acquired, followed by slower improvement as capability develops.

The evolution of the EDL trajectory during training reflects this structure. Early training accumulates MDL rapidly (format not yet learned, so outputs are surprising) but also sees rapid test loss improvement (format quickly acquired). This yields an initial spike in EDL followed by decreasing per-example returns to EDL despite significant training. As format learning completes, subsequent examples address capability. MDL accumulation may slow (format already known) or continue at a moderate rate (capability still developing), while test loss improvement also slows. EDL per example may increase during this phase as capability develops.

\begin{proposition}
\label{prop:formatlearning}
    When task requires both format and capability learning, EDL per example may initially decrease (format learning) before increasing (capability learning) or stabilizing (capability saturation).
\end{proposition}

This model motivates the practical recommendation to score only designated output tokens (\textit{e.g.}, answer tokens) or normalize outputs when computing EDL for capability analysis. Scoring the full or unnormalized output, including format elements, confounds format learning with capability acquisition and can obscure important signatures of interest.

\begin{figure}[ht]
    \centering
        \includegraphics[width=1.0\columnwidth]{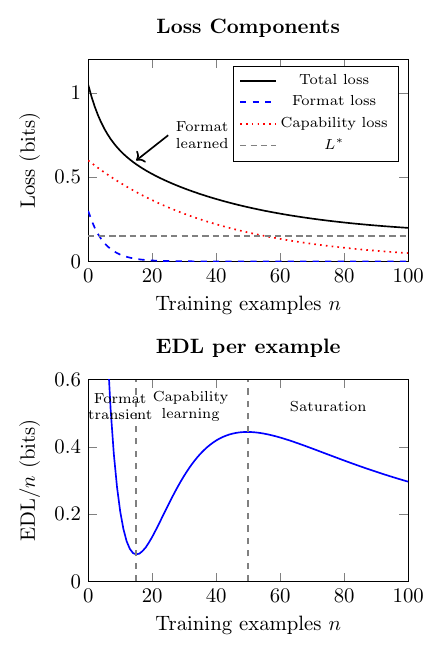}
    \caption{\textbf{Format learning can induce transients in EDL scaling and dynamical behavior.} In teaching scenarios, format learning can cause a transient drop in the per-example EDL ($\text{EDL}/n$), as format comprises generalizable structure that can be easily learned before task capability is acquired.}
    \label{fig:FormatLearningDiagram}
\end{figure}

\iftoggle{workingversion}{\clearpage}{}
\section{Conceptual Clarifications}
\label{sec::ConceptualClarifications}

This section addresses predictable confusions about the interpretation of EDL and its relationship to other notions of information. Deconflating these perspectives and clarifying their subtleties is essential for correctly interpreting empirical measurements.

\subsection{Computed Quantities Versus Distributional Expectations}
\label{sec::DistributionalExpectations}

The prequential MDL is computed on a specific realized sequence of training examples $D = \{(x_i, y_i)\}_{i=1}^n$, and describes the amount of information in those data via the minimum compressed file size necessary to optimally encode them with the chosen learning algorithm. However, in what sense does this computed quantity represent ``information''?

Information-theoretic quantities such as entropy and mutual information are defined as expectations over distributions. The entropy of a random variable $Y$ is $H(Y) = \mathbb{E}[-\log P(Y)]$, an expected value; it is not a property of any particular realization of the outputs of some data-generating procedure. Similarly, the ``information'' conveyed by observing a specific value $y$ is typically defined through the expected reduction in uncertainty conditional on that observation.

The prequential MDL is a sum of realized losses, not expected losses. For a single dataset, it is a random variable depending on the specific examples that were drawn and the order in which they were encoded with the model. Its interpretation as the ``information content of the data'' is most rigorous when considering the expectation over datasets: $\mathbb{E}_D[\text{MDL}(D; \theta_0, A)]$ measures the expected codelength of a sequence of data sampled from the distribution and has clear information-theoretic meaning in a source coding sense.

For a single computed EDL value, the appropriate interpretation is as an \textit{estimate} of the expected generalizable information a model can gain from the data during training. Just as a sample mean estimates a population mean, the computed EDL estimates the expected information that would be absorbed across datasets drawn from $\mathcal{D}$. Variance across datasets (or across random orderings of a fixed dataset) provides a measure of the uncertainty in this estimate.

This perspective suggests reporting confidence intervals or variance estimates alongside EDL measurements, particularly for small datasets where variance may be substantial.

\subsection{One Label Can Provide Many Bits}
\label{sec::OneLabelManyBits}

A binary label has Shannon entropy at most 1 bit. How, then, can a single example provide more than 1 bit of information to a learner?

Here, the resolution lies in distinguishing the entropy of the label itself from the information the label provides about the learning problem (\textit{i.e.}, data distribution). The label's entropy measures uncertainty about the individual label itself, in isolation. The information about the learning problem measures uncertainty reduction about which hypothesis (or parameter setting) governs the data.

When a learner has $m$ hypotheses under consideration, eliminating all but a single hypothesis provides $\log m$ bits of information about which hypothesis is correct, regardless of the label's own entropy. This information (what is specifically gained about the overall distribution by eliminating hypotheses) is what EDL captures and represents: the reduction in the model's uncertainty about how to predict future examples.

In formal terms, consider the mutual information between the label $Y$ and the hypothesis $h$ governing the data: $I(Y; h)$. This can be as large as $H(h) = \log m$, even when $H(Y)$ is small. The label may have low marginal entropy but high mutual information with the hypothesis.

This clarifies how elicitation can occur with minimal data. If the model's uncertainty is over which of its latent capabilities to apply (a small hypothesis space, as the model already has the requisite capability), a single illustrative example can resolve this uncertainty entirely. The example provides many bits of information about how to respond to the task, not by conveying information about the label itself, but by identifying which latent competency is relevant.\footnote{EDL quantifies the amount of information specifically present in the train dataset that was gained by the model and used to determine how to properly respond to the task. As such, the maximum amount of this generalizable information that a model can directly extract from the train data is the total information content of the train set itself.}

\subsection{EDL Is Not Semantic Information}
\label{sec::EDLvsSemanticInformation}

EDL measures predictive compression, \textit{i.e.}, how much better the final model predicts test labels compared to encoding labels sequentially using the continually updating model being trained. This is a statistical notion of information, rather than a semantic one.

A model with high EDL is better at predicting test labels. This improved prediction may arise from ``understanding'' the task in some meaningful sense, or it may arise from pattern matching, memorization, or exploitation of statistical regularities. However, EDL is agnostic to the mechanism by which this is accomplished; it measures the predictive improvement but does not reveal how that improvement was achieved.

Similarly, EDL does not measure ``knowledge'' in any philosophical sense. A model might absorb information which is sufficient to predict labels correctly, while having no meaningful representation of the underlying concepts. Conversely, a model might ``understand'' a domain deeply but show low EDL if that understanding was already evident before fine-tuning (elicitation or refinement of a demonstrated capability) or if the test distribution differs from training (distribution shift).

These distinctions matter for interpreting EDL measurements. It is important to remember that EDL reflects the change in a model's ability to identify regularities in the data and to make better predictions based on this recognition. High EDL indicates that fine-tuning has substantially improved predictions, but it does not reveal whether the improvement reflects genuine capability acquisition or more superficial explanations. Low EDL in an elicitation scenario indicates that the capability was largely manifest beforehand, but it does not reveal whether the capability reflects ``true understanding'' or sophisticated pattern matching.

Ultimately, EDL is a tool for measuring information flow during training, as opposed to one for revealing the nature of model representations.

\subsection{Connection to Parameter Capacity}
\label{sec::ParameterCapacity}

A natural question is how EDL relates to the number of parameters in the model, particularly for parameter-efficient fine-tuning methods that update only a small (sub)set of parameters.

Each parameter, represented in floating-point precision, has a theoretical upper bound on information content determined by its bit-length (e.g., 32 bits for float32). However, this bound is rarely approached in practice. The effective information capacity of a parameter depends on how it influences predictions, which is constrained by the loss landscape, optimization dynamics, and the model's architecture.

EDL per trainable parameter, $\text{EDL}_{\text{param}} = \text{EDL}/P$, where $P$ is the number of trainable parameters, provides an empirical measure of how much additional predictive information has been gained per parameter during training. For full fine-tuning with many parameters, this ratio is typically small; models tend to be over-parameterized, and many parameters may not store task-specific information. For parameter-efficient methods like LoRA with limited trainable parameters, this ratio can be larger, potentially approaching or exceeding 1 bit per parameter.

Empirically, we observe a threshold in bits per trainable parameter beyond which parameter-efficient methods degrade relative to full fine-tuning. When the task requires storing more information than the adapter's parameters can accommodate, performance suffers. This parameter capacity threshold differs systematically between elicitation (lower threshold, \addtexttilde0.01--0.1 bits/parameter) and teaching (higher threshold, \addtexttilde1+ bits/parameter), reflecting the different information requirements of these regimes (see companion empirical paper).

These thresholds are empirical observations; the theoretical basis for such thresholds, and their precise values, remain open questions  (see \autoref{sec::OpenQuestions}). The two-order-of-magnitude difference between regimes suggests a fundamental distinction in information requirements, which we formalize through the toy models in \autoref{sec::ToyModels}. The connection between EDL per parameter and adapter capacity provides a promising direction for understanding the scaling behavior of parameter-efficient fine-tuning.

\iftoggle{workingversion}{\clearpage}{}
\section{Discussion}
\label{sec::discussion}

We have developed a formal framework for measuring the generalizable portion of train set information that fine-tuning absorbs into model parameters. Excess Description Length, defined via prequential coding, provides an operational quantity computable from standard training logs. We established that EDL is non-negative in expectation, converges to surplus description length asymptotically, and provides bounds on expected generalization improvement. Through toy models, we clarified when EDL signatures indicate elicitation (pre-existing capability; few bits needed) versus teaching (capability initially absent; many bits needed).

\subsection{Implications for Empirical Research}
\label{sec::EmpiricalResearchImplications}

The framework developed here provides foundations for empirical investigation of teaching versus elicitation. These learning regimes are characterized by qualitatively different EDL behavior:

For elicitation, EDL per example monotonically decreases with dataset size (diminishing marginal generalizable information gained from additional examples), total EDL is small relative to capability improvement (capability was readily accessible or mostly present), and the bits-per-parameter threshold for adapter degradation relative to full fine-tuning is low (little new or remaining structure in the data to encode).

For teaching, EDL per example may increase during learning (new capability and associated representations being formed), total EDL is large relative to capability improvement (new structure being encoded in the model parameters), and the bits-per-parameter threshold is higher (substantial new structure required).

These signatures, formalized through our framework, provide testable predictions for empirical study. Causal interventions, such as pre-training on related tasks before measuring EDL on a target task, can shift EDL signatures from teaching-like to elicitation-like (shifting capacity thresholds by more than an order of magnitude and inverting EDL scaling behavior), providing evidence for the distinction of latent versus absent capabilities.

\subsubsection{Empirical Validation Preview}

The framework developed here makes specific predictions that we validate empirically in a companion paper:

\begin{enumerate}
    \item EDL scaling signatures: Elicitation shows monotonically decreasing EDL/token with dataset size; teaching shows an initial increasing phase. These signatures are robust across models (Llama 3 1B--8B, TinyStories-1B, Qwen2.5 1.5B--14B) and tasks (arithmetic, reasoning).
    \item Causal intervention: Pre-teaching a skill (\textit{e.g.}, multiplication via operator notation) converts a teaching task to an elicitation task, reducing information thresholds by \addtexttilde10--100x.
    \item Random label control: EDL collapses to near zero when labels are randomly permuted, confirming that EDL measures learnable structure rather than training effort.
    \item Parameter capacity prediction: Performance degrades significantly when EDL per parameter exceeds regime-dependent thresholds, with elicitation thresholds (\addtexttilde 0.01--0.1 bits/parameter) roughly 100x lower than teaching thresholds (>1 bit/parameter).
\end{enumerate}

These empirical results support the theoretical framework we introduce here and demonstrate practical utility for predicting fine-tuning behavior.

\subsection{Limitations}
\label{sec::Limitations}

Several limitations of our framework deserve acknowledgment.

First, EDL depends on the training algorithm. This dependence is real and meaningful---different algorithms genuinely differ in learning efficiency---but this complicates comparison across settings. The choice of algorithm should be held fixed when comparing EDL across models or tasks.

Second, EDL is not a semantic measure of knowledge or understanding. Large EDL indicates improved prediction but does not reveal how improvement was achieved. Models with similar EDL may have very different internal representations.

Third, our analysis focuses on supervised fine-tuning. Extensions to reinforcement learning, preference optimization, and other post-training methods require additional development.

Fourth, the toy models, while illuminating, are highly simplified. Real learning involves continuous hypothesis spaces, complex coverage dynamics, and entanglement of format and capability that our discrete toy models do not fully capture.

\subsection{Open Questions}
\label{sec::OpenQuestions}

The EDL formalism we introduce here provides a preliminary framework for assessing capability emergence and describing different learning regimes. There are several directions and open questions which we believe merit further investigation.

Can we establish tighter bounds relating EDL to generalization? Our current bounds are loose; sharper characterization would strengthen the framework.

How does information distribute across model components during fine-tuning? Layer-wise or module-wise analysis of EDL accumulation could reveal where capability information is stored.

Can EDL scaling reliably predict capability ceilings? If the trajectory of EDL with dataset size has predictable asymptotic behavior, it may be possible to estimate maximum achievable performance from early training data.

What training algorithms optimize information extraction? If EDL measures absorbed information, are there algorithms that maximize EDL (and thus learning efficiency) for a given data budget?

\subsection{Related Work}
\label{sec::RelatedWork}

Our methodology is grounded in an information-theoretic view of learning, using the Minimum Description Length (MDL) principle to quantify the information models absorb from data \cite{rissanen1978modeling, Dawid1984prequentialstats, Barron1998prequentialmdl}. This approach is motivated by prior work in information-theoretic probing. Notably, \citet{blier2018description} and \citet{voita2020information} also decomposed the MDL of a dataset into a model component and a data component, finding that the model component reflects the information captured by the parameters after training. Their formulation is a conceptual precursor to our definition of Excess Description Length (EDL).

Our framework also builds upon the concept of Surplus Description Length (SDL) \citep{SDLwhitney2021evaluatingrepresentationscomplexitylearning}, a powerful, information-theoretic measure of the quality of learned representations. However, SDL is typically defined in an asymptotic, infinite-data setting. A central contribution of our work is to formulate EDL as a practical, finite-data analog to SDL. While EDL can converge to SDL in the infinite-data limit, its primary advantage is in analyzing the realistic, resource-constrained scenarios common in fine-tuning. EDL is designed to answer practical questions about how far a small, finite amount of information can be stretched, providing a tool to quantify learning dynamics where data is scarce. This allows us to connect our formal metrics to an intuitive understanding of elicitation and teaching in a way that asymptotic measures cannot.

Our framework connects to a rich literature on information-theoretic analysis of learning. \citet{xu2017information} established generalization bounds based on the mutual information between training data and learned parameters, showing that algorithms with bounded information complexity generalize well. \citet{steinke2020reasoning} extended this analysis to adaptive data analysis, where queries depend on previous answers. The PAC-Bayes framework \citep{mcallester1999pacbayesian,catoni2007pacbayesian} provides generalization bounds based on the KL divergence between posterior and prior over hypotheses, which can be viewed as a description length. Our EDL framework differs in focusing on the prequential coding perspective, which provides operational semantics through the interactive communication protocol, as well as in emphasizing the finite-data regime where teaching and elicitation exhibit qualitatively different signatures.

\subsection{Conclusion}
\label{sec::Conclusion}

Understanding how fine-tuning modifies language models, whether surfacing latent capabilities or imparting new ones, is essential for capability evaluation and safety. The framework developed here provides rigorous foundations for this investigation, offering operational definitions, proven properties, and conceptual clarifications. We hope this work enables more precise empirical study of the teaching-versus-elicitation distinction and contributes to our understanding of what language models learn from fine-tuning.

\section{Acknowledgments}

We thank John Schulman, Ethan Perez, Eric Easley, and Vivek Hebbar for helpful discussions.

E.D. conceived the Excess Description Length (EDL) framework, developed the theoretical formulation and proofs, designed and analyzed the toy models, and wrote the manuscript. E.D. also designed and conducted the empirical experiments described in the companion empirical paper that motivated and validated the theoretical framework.

F.R. proposed the disjoint subdistributions setting. H.J. and F.R. provided feedback on the theoretical development and interpretation, drafts, and contributed to discussion of the framework and its implications. J.L. supervised the project and provided guidance and feedback throughout.

\small
\bibliography{bibliography}
\iftoggle{icml}{\bibliographystyle{ext/icml/icml2026}}

\appendix
\include{appendix}
\end{document}

%% file: notation.tex
\newcommand{\groupname}[1]{#1}
\definecolor{best}{HTML}{BAFFCD}
\definecolor{bad}{HTML}{FFC8BA}
\newcommand{\violation}[1]{\cellcolor{bad}#1} 
\newcommand{\good}[1]{{#1} }
\newcommand{\badvalue}[1]{{\color{red}{\textbf{#1}}}} 

\newcommand{\predictionthresholds}[0]{\cell{l}{$\tau=0.05$\\ $\tau=0.1$\\ $\tau=0.15$\\ $\tau=0.2$}}

\definecolor{fgood}{HTML}{BAFFCD}
\definecolor{fbad}{HTML}{FFC8BA}
\definecolor{funk}{HTML}{FFFFE0}
\definecolor{darkfunk}{HTML}{DDDD00}
\newcommand{\ftblmidrules}{\hline}
\newcommand{\ftblgap}{\quad}
\newcommand{\ftblheader}[3]{\multicolumn{#1}{#2}{\cell{#2}{#3}}}
\newcommand\setrow[1]{\gdef\rowmac{#1}#1\ignorespaces}
\newcommand\clearrow{\global\let\rowmac\relax}

\newcommand{\methodname}{conceptual safeguards}

\newcommand{\argmax}{\operatorname*{arg\,max}}
\newcommand{\argmin}{\operatorname*{arg\,min}}
\newcommand{\R}{\mathbb{R}}
\newcommand{\E}[1]{\mathbb{E}(#1)}
\newcommand{\indic}[1]{\mathbb{I}(#1)}
\renewcommand{\Pr}[1]{\textrm{Pr}(#1)}
\newcommand{\PrHat}[1]{\widehat{\textrm{Pr}}(#1)}
\newcommand{\intrange}[1]{[#1]}
\newcommand{\pypred}{\hat{p}_{y \mid \xb}}

\newcommand{\bigCI}{\mathrel{\text{\scalebox{1.07}{$\perp\mkern-10mu\perp$}}}}

\newcommand{\textds}[1]{{\texttt{\footnotesize{#1}}}} 
\newcommand{\textcp}[1]{{\texttt{\footnotesize{#1}}}} 
\newcommand{\textmethod}[1]{{\textsf{\footnotesize{#1}}}}

\newcommand{\X}{\mathcal{X}}
\newcommand{\xb}{\bm{x}}
\newcommand{\C}{\mathcal{C}}
\newcommand{\nconcepts}{m}
\newcommand{\cb}{\bm{c}}
\newcommand{\cbhat}{\hat{\bm{c}}}
\newcommand{\ctrue}{\bm{c}}
\newcommand{\Cpredproba}{\bm{\hat{C}}}
\newcommand{\cpredproba}{\bm{\hat{c}}}
\newcommand{\CHAT}{\widehat{\mathcal{C}}}

\newcommand{\Y}{\mathcal{Y}}
\newcommand{\YHAT}{\widehat{\mathcal{Y}}}
\newcommand{\ytrue}{y} 
\newcommand{\ypred}{\hat{y}} 
\newcommand{\missing}{?}

\newcommand{\fxc}[1]{g_{#1}} 
\newcommand{\Fset}{\mathcal{F}}
\newcommand{\Gset}{\mathcal{G}}

\newcommand{\truerisk}[2]{{R_{#1}(#2)}}
\newcommand{\emprisk}[2]{{\hat{R}_{#1}(#2)}}
\newcommand{\loss}{\ell}

\newcommand{\verify}{\pi}
\newcommand{\verifyset}{S}
\newcommand{\gate}{\varphi}
\newcommand{\intinds}{S}
\newcommand{\ip}{\pi}
\newcommand{\cintproba}{\bm{\bar{c}}}
\newcommand{\Cintproba}{\bm{\bar{C}}}
\newcommand{\cost}[1]{\gamma_{#1}}
\newcommand{\intcost}{\texttt{C}}
\newcommand{\allintinds}{\mathcal{S}}
\newcommand{\mccost}{\texttt{c}_\ell}
\newcommand{\thresh}{R^\textrm{max}}
\newcommand{\budget}{B}
\newcommand{\tol}{\alpha}
\newcommand{\threshold}{\tau}

\newcommand{\abstain}{\perp}
\newcommand{\absfun}[1]{\optpar{h}{#1}}
\newcommand{\absrate}[1]{A({#1})}
\newcommand{\abstcost}{\texttt{c}_\abstain}

\newcommand{\Parg}[1]{\mathrm{Pr}\left(#1\right)}
\newcommand{\Earg}[1]{\mathbb{E}\left[#1\right]}
\newcommand{\Esarg}[2]{\mathbb{E}_{#1}\left[#2\right]}
\newcommand{\probhat}{\hat{\mathrm{Pr}}}
\newcommand{\defeq}{:=}

\newcommand{\accuracy}{\mathrm{Accuracy}}
\newcommand{\coverage}{\mathrm{Coverage}}

\newcommand{\features}{\mathbf{x}}
\newcommand{\feature}{x}

\newcommand{\concepts}{\mathbf{c}}
\newcommand{\concept}{c}

\newcommand{\outcome}{y}

\newcommand{\featurespace}{\mathcal{X}}
\newcommand{\conceptspace}{\mathcal{C}}
\newcommand{\outcomespace}{\mathcal{Y}}
\newcommand{\reals}{\mathbb{R}}

\newcommand{\pipeline}{h}
\newcommand{\detector}{g}
\newcommand{\frontend}{f}
\renewcommand{\verify}{\psi}
\renewcommand{\gate}{\varphi}

\newcommand{\conceptspred}{\mathbf{q}}
\newcommand{\conceptpred}{q}

\newcommand{\verifieds}{\boldsymbol{p}}
\newcommand{\verified}{p}

\newcommand{\outcomepredproba}{\overline{y}}
\newcommand{\outcomepred}{\hat{\outcome}}

%% file: appendix.tex
\section{Proofs of Main Results}
\label{Sec::AppendixMainResultsProofs}

\subsection{Proof of \autoref{thm:nonneg} (Non-Negativity in Expectation)}
\label{Sec::AppendixNonNegativityProof}

\begin{theorem}[Reproduced from \autoref{thm:nonneg}]
    Let $D$ be drawn i.i.d. from distribution $\mathcal{D}$, let the test loss be evaluated on an independent sample $D_\text{test}$ from $\mathcal{D}$, and let $A$ be a population-monotonic algorithm (\autoref{def:populationmonotonic}). Then
    \begin{equation}
        \mathbb{E}_{D, D_{\mathrm{test}}}[\mathrm{EDL}(D; \theta_0, A)] \geq 0.
    \end{equation}
\end{theorem}

\begin{proof}
    \label{proof:nonneg}
    We work with expectations over the joint distribution of training data $D \sim \mathcal{D}^n$ and test data $D_\text{test}$.

    Define $L(\theta) = \mathbb{E}_{(x,y) \sim \mathcal{D}}[\ell(\theta; x, y)]$ as the population loss for parameters $\theta$. The test loss satisfies $\mathbb{E}_{D_{\text{test}}}[L_{\text{test}}(\theta^*) \mid \theta^*] = L(\theta^*)$ by independence of test data from training data.
    
    For the MDL term, we condition on the filtration $\mathcal{F}_{i-1} = \sigma(\theta_0, x_1, y_1, \ldots, x_{i-1}, y_{i-1})$, which determines $\theta_{i-1}$. Since $(x_i, y_i) \sim \mathcal{D}$ independently of $\mathcal{F}_{i-1}$,
    \begin{equation}
        \mathbb{E}[\ell(\theta_{i-1}; x_i, y_i) \mid \mathcal{F}_{i-1}] = L(\theta_{i-1}).
    \end{equation}
    
    By the law of total expectation:
    \begin{equation}
        \mathbb{E}[\text{MDL}] = \mathbb{E}\left[\sum_{i=1}^n \ell(\theta_{i-1}; x_i, y_i)\right] = \sum_{i=1}^n \mathbb{E}[L(\theta_{i-1})].
    \end{equation}
    
    For the residual test loss term:
    \begin{equation}
        \mathbb{E}[n \cdot L_{\text{test}}(\theta^*)] = n \cdot \mathbb{E}[L(\theta^*)].
    \end{equation}
    
    Therefore,
    \begin{equation}
        \mathbb{E}[\text{EDL}] = \sum_{i=1}^n \mathbb{E}[L(\theta_{i-1})] - n \cdot \mathbb{E}[L(\theta^*)] = \sum_{i=1}^n \mathbb{E}[L(\theta_{i-1}) - L(\theta^*)].
    \end{equation}
    
    It remains to show that $\mathbb{E}[L(\theta_{i-1}) - L(\theta^*)] \geq 0$ for each $i \in \{1, \ldots, n\}$.

    By the population-monotonicity condition (\autoref{def:populationmonotonic}), for each update step $j$, we have $\mathbb{E}[L(\theta_j) \mid \theta_{j-1}] \leq L(\theta_{j-1})$. We now prove by induction that $\mathbb{E}[L(\theta^*) \mid \theta_{i-1}] \leq L(\theta_{i-1})$.

    \textit{Base case:} For $i = n$, $\theta^*$ is obtained from $\theta_{n-1}$ by training on $(x_n, y_n)$ and potentially additional epochs. By population-monotonicity applied iteratively:
    \begin{equation}
        \mathbb{E}[L(\theta^*) \mid \theta_{n-1}] \leq L(\theta_{n-1}).
    \end{equation}

    \textit{Inductive step:} Suppose $\mathbb{E}[L(\theta^*) \mid \theta_{j-1}] \leq L(\theta_{j-1})$ for all $j \geq i$. For $j = i-1$:
    \begin{equation}
    \begin{split}
        \mathbb{E}[L(\theta^*) \mid \theta_{i-1}] &= \mathbb{E}\left[\mathbb{E}[L(\theta^*) \mid \theta_i] \mid \theta_{i-1}\right] \\
        &\leq \mathbb{E}[L(\theta_i) \mid \theta_{i-1}] \quad (\text{by inductive hypothesis})\\
        &\leq L(\theta_{i-1}) \quad (\text{by population-monotonicity}).
    \end{split}
    \end{equation}
    
    Taking expectations over $\theta_{i-1}$:
    \begin{equation}
        \mathbb{E}[L(\theta^*)] \leq \mathbb{E}[L(\theta_{i-1})],
    \end{equation}
    which implies $\mathbb{E}[L(\theta_{i-1}) - L(\theta^*)] \geq 0$.
    
    Since each term in the sum is non-negative, $\mathbb{E}[\text{EDL}] \geq 0$.
\end{proof}

\begin{remark}
    The assumption that training does not increase expected loss is mild and holds for standard optimizers on convex losses or with appropriate learning rates on non-convex losses. The proof can be tightened for specific algorithm families with known convergence guarantees.
\end{remark}

\subsection{Proof of \autoref{thm:MDLRegret} (MDL and Regret)}
\label{Sec::AppendixMDLRegretProof}

\begin{theorem}[Reproduced from \autoref{thm:MDLRegret}]
    For any fixed $\theta \in \Theta$,
    \begin{equation}
        \mathrm{MDL}(D; \theta_0, A) = \sum_{i=1}^n \ell(\theta; x_i, y_i) + R_n(\theta),
    \end{equation}
    where $R_n(\theta) = \sum_{i=1}^n \ell(\theta_{i-1}; x_i, y_i) - \sum_{i=1}^n \ell(\theta; x_i, y_i)$ is the regret relative to $\theta$.
\end{theorem}

\begin{proof}
    This follows directly from the definition of regret. By definition:
    \begin{equation}
        R_n(\theta) = \sum_{i=1}^n \ell(\theta_{i-1}; x_i, y_i) - \sum_{i=1}^n \ell(\theta; x_i, y_i).
    \end{equation}

    Rearranging:
    \begin{equation}
        \sum_{i=1}^n \ell(\theta_{i-1}; x_i, y_i) = \sum_{i=1}^n \ell(\theta; x_i, y_i) + R_n(\theta).
    \end{equation}

    The left-hand side is precisely $\text{MDL}(D; \theta_0, A)$.
\end{proof}

\begin{corollary}
    Taking $\theta = \theta^*$ (the final trained parameters):
    \begin{equation}
        \mathrm{MDL} = \sum_{i=1}^n \ell(\theta^*; x_i, y_i) + R_n(\theta^*).
    \end{equation}
    
    If the final model achieves training loss $L_{\text{train}}(\theta^*) = \frac{1}{n}\sum_{i=1}^n \ell(\theta^*; x_i, y_i)$, then:
    \begin{equation}
        \mathrm{MDL} = n \cdot L_{\text{train}}(\theta^*) + R_n(\theta^*).
    \end{equation}
\end{corollary}

\begin{corollary}
    For algorithms with sublinear regret $R_n(\theta^*) = o(n)$:
    \begin{equation}
        \frac{\mathrm{MDL}}{n} \to L_{\text{train}}(\theta^*) \quad \text{as } n \to \infty.
    \end{equation}
\end{corollary}

\subsection{Proof of \autoref{thm:ConvToSDL} (Asymptotic Convergence to Surplus Description Length)}
\label{Sec::AppendixSDLConvergenceProof}

\begin{theorem}[Reproduced from \autoref{thm:ConvToSDL}]
    \label{thm:ConvToSDL}
    Suppose the learning algorithm is consistent for the model class $\Theta$ under distribution $\mathcal{D}$: $L(\theta^*) \to L^*$ almost surely as $n \to \infty$, where $L^* = \inf_{\theta \in \Theta} L(\theta)$. Then
    \begin{equation}
        \frac{\mathrm{EDL} - \mathrm{SDL}}{n} = n \rightarrow 0 \quad \text{as} \quad n \rightarrow \infty.
    \end{equation}
    Specifically,
    \begin{equation}
        \frac{\mathrm{EDL} - \mathrm{SDL}}{n} = n \cdot \left(L^* - L_\mathrm{test}(\theta^*)\right),
    \end{equation}
    and under consistency, this difference grows sublinearly in $n$.
\end{theorem}
\begin{proof}
    By definition:
    \begin{equation}
    \begin{split}
        \text{SDL}(D; A) &= \text{MDL}(D; \theta_0, A) - n \cdot L^*,\\
        \text{EDL}(D; \theta_0, A) &= \text{MDL}(D; \theta_0, A) - n \cdot L_{\text{test}}(\theta^*).
    \end{split}
    \end{equation}

    Subtracting:
    \begin{equation}
        \begin{split}
            \text{EDL} - \text{SDL} &= n \cdot L^* - n \cdot L_{\text{test}}(\theta^*)\\
            &= n \cdot (L^* - L_{\text{test}}(\theta^*)).
        \end{split}
    \end{equation}

    Since $L_{\text{test}}(\theta^*)$ is an unbiased estimator of $L(\theta^*)$ (for large test sets, or in expectation), and $L(\theta^*) \geq L^*$ by definition of $L^*$:
    \begin{equation}
        \text{EDL} - \text{SDL} = n \cdot (L^* - L(\theta^*)) + n \cdot (L(\theta^*) - L_{\text{test}}(\theta^*)).
    \end{equation}

    The second term is $O(\sqrt{n})$ in probability by concentration of the test loss estimator.

    For the first term, consistency implies $L(\theta^*) - L^* \to 0$ as $n \to \infty$. The rate at which $L(\theta^*) - L^* \rightarrow 0$ depends on the model class and is not specified by consistency alone. For the purposes of this theorem, we require only that the convergence occurs, not a specific rate.

    The important observation is that,
    \begin{equation}
        \text{EDL} - \text{SDL} = n \cdot \left(L^* - L(\theta^*)\right) + O(\sqrt{n}),
    \end{equation}
    where the $O(\sqrt{n})$ term comes from concentration of the test loss estimator.

    Under consistency, $L(\theta^*) - L^* = o(1)$, so:
    \begin{equation}
        \frac{|\text{EDL} - \text{SDL}|}{n} = \left|L^* - L(\theta^*)\right| + O(\sfrac{1}{\sqrt{n}}) \to 0.
    \end{equation}
\end{proof}

\subsection{Proof of \autoref{thm:genbound} (Generalization Bound)}
\label{Sec::AppendixGeneralizationBoundProof}

We first restate \autoref{thm:genbound} in a cleaner form.

\begin{theorem}[Restated from \autoref{thm:genbound}]
    The expected EDL decomposes as:
    \begin{equation}
        \mathbb{E}[\text{EDL}] = n \cdot (\bar{L} - \mathbb{E}[L(\theta^*)]),
    \end{equation}
    where $\bar{L} = \frac{1}{n}\sum_{i=1}^n \mathbb{E}[L(\theta_{i-1})]$ is the average expected population loss during training.

    Equivalently, the expected improvement in population loss satisfies:
    \begin{equation}
        \mathbb{E}[L(\theta^*)] = \bar{L} - \frac{\mathbb{E}[\text{EDL}]}{n}.
    \end{equation}
\end{theorem}

\begin{proof}
    From the proof of \autoref{thm:nonneg}:
    \begin{equation}
        \mathbb{E}[\text{EDL}] = \sum_{i=1}^n \mathbb{E}[L(\theta_{i-1})] - n \cdot \mathbb{E}[L(\theta^*)].
    \end{equation}

    Dividing the sum by $n$ and rearranging:
    \begin{equation}
    \begin{split}
        \mathbb{E}[\text{EDL}] &= n \cdot \bar{L} - n \cdot \mathbb{E}[L(\theta^*)]\\
        &= n \cdot (\bar{L} - \mathbb{E}[L(\theta^*)]).
    \end{split}
    \end{equation}

    Solving for $\mathbb{E}[L(\theta^*)]$:
    \begin{equation}
        \mathbb{E}[L(\theta^*)] = \bar{L} - \frac{\mathbb{E}[\text{EDL}]}{n}.
    \end{equation}
\end{proof}

\begin{corollary}
    The total improvement from initial to final model decomposes as:
    \begin{equation}
        L(\theta_0) - \mathbb{E}[L(\theta^*)] = (L(\theta_0) - \bar{L}) + \frac{\mathbb{E}[\mathrm{EDL}]}{n}.
    \end{equation}
\end{corollary}

The first term $(L(\theta_0) - \bar{L})$ represents the average improvement during the training trajectory. The second term $\mathbb{E}[\text{EDL}]/n$ represents the additional improvement captured by the final model relative to the training-averaged model, which is the per-example absorbed information.

\subsection{Proof of \autoref{prop:algdependence} (Algorithm Dependence)}
\label{Sec::AppendixAlgorithmDependenceProof}

\begin{proposition}[Restated from \autoref{prop:algdependence}]
    For fixed $D$ and $\theta_0$, different algorithms $A$ and $A'$ can yield different EDL values.
\end{proposition}

\begin{proof}
    We construct an explicit example. Consider a simple linear regression setting with $D = \{(x_i, y_i)\}_{i=1}^n$ where $x_i \in \mathbb{R}^d$ and $y_i \in \mathbb{R}$. Let the model be $p_\theta(y|x) = \mathcal{N}(\theta^\top x, \sigma^2)$ for fixed $\sigma^2$.

    Consider two algorithms:
    \begin{itemize}
        \item $A$: Gradient descent with learning rate $\eta$
        \item $A'$: Gradient descent with learning rate $\eta' = \eta/10$
    \end{itemize}

    Both start from the same $\theta_0$. Algorithm $A$, with larger learning rate, makes faster initial progress, achieving lower loss earlier in training. Algorithm $A'$ makes slower progress.

    For the MDL computation (losses during first epoch):
    \begin{itemize}
        \item Under $A$: losses decrease rapidly, so later examples have lower loss; total MDL is moderate.
        \item Under $A'$: losses decrease slowly, so more examples have high loss; total MDL is higher.
    \end{itemize}

    Let $\text{MDL}_A$ and $\text{MDL}_{A'}$ denote the respective MDL values. We have $\text{MDL}_A < \text{MDL}_{A'}$ in general.

    For the final test loss:
    \begin{itemize}
        \item If both algorithms converge to the same optimum (as they would for convex losses with sufficient training), then $L_{\text{test}}(\theta^*_A) = L_{\text{test}}(\theta^*_{A'})$.
        \item If training is truncated at a fixed number of steps, $A$ may achieve lower test loss than $A'$.
    \end{itemize}

    In either case:
    \begin{equation}
    \begin{split}
        \text{EDL}_A &= \text{MDL}_A - n \cdot L_{\text{test}}(\theta^*_A)\\
        &\neq \text{MDL}_{A'} - n \cdot L_{\text{test}}(\theta^*_{A'}) = \text{EDL}_{A'}.
    \end{split}
    \end{equation}

    The difference arises from different MDL values (different loss trajectories) and potentially different final test losses.
\end{proof}

\section{Detailed Analysis of Toy Models}
\label{Sec::ToyModelsAppendix}

\subsection{Random Labels (\autoref{sec::RandomLabels})}
\label{Sec::RandomLabelsAppendix}

\textit{Setup.} The label $Y$ is independent of input $X$ and drawn uniformly: $Y \sim \text{Uniform}(\mathcal{Y})$ with $|\mathcal{Y}| = k$. Training labels $\{y_i\}_{i=1}^n$ and test labels are drawn independently from this marginal.

\begin{proposition}[Reproduced from \autoref{prop::randomlabels}]
    For i.i.d. random labels independent between train and test, $\mathbb{E}[\text{EDL}] = 0$.
\end{proposition}

\begin{proof}
    The optimal predictor for this problem is $p_\text{opt}(y|x) = 1/k$ for all $x$, achieving loss $L(\theta_\text{opt}) = \log k$ per example (the Bayes error rate). Since labels carry no information about $X$, no other predictor can achieve lower expected loss.
    
    Consider a model that maintains the uniform prediction throughout training (the Bayes-optimal strategy upon recognizing random labels). Then:
    \begin{equation}
        \ell(\theta_i; x, y) = \log k \quad \text{for all } i, x, y.
    \end{equation}
    The MDL is $\text{MDL} = n \log k$.
    
    The test loss is $L_{\text{test}}(\theta^*) = L_{\text{test}}(\theta_\text{opt}) = \log k$, since the optimal predictor achieves this loss.

    Therefore:
    \begin{equation}
        \text{EDL} = n \log k - n \log k = 0.
    \end{equation}

    Now consider a model that attempts to memorize training labels. During training, the model may achieve lower loss on previously-seen examples, but for the first exposure to each example (which determines MDL), the model cannot predict better than chance. Thus $\text{MDL} \geq n \log k$ with equality if the model starts near-uniform.
    
    After training, the memorized labels do not help predict independent test labels, so $L_{\text{test}}(\theta^*) \geq \log k$.
    
    In expectation over training and test data:
    \begin{equation}
    \begin{split}
        \mathbb{E}[\text{MDL}] &\geq n \log k,\\
        \mathbb{E}[L_{\text{test}}(\theta^*)] &\geq \log k.
    \end{split}
    \end{equation}

    For a well-calibrated model that recognizes the futility of memorization:
    \begin{equation}
    \begin{split}
        \mathbb{E}[\text{MDL}] &= n \log k,\\
        \mathbb{E}[L_{\text{test}}(\theta^*)] &= \log k,
    \end{split}
    \end{equation}
    giving $\mathbb{E}[\text{EDL}] = 0$.

    For a model that memorizes, MDL may be slightly higher (if memorization interferes with predictions on new training examples during the first epoch), but test loss will also be at least $\log k$. The net effect is $\mathbb{E}[\text{EDL}] \approx 0$.

    More formally, by \autoref{thm:nonneg}, $\mathbb{E}[\text{EDL}] \geq 0$. By the optimality of uniform predictions for random labels, no model can achieve $L_{\text{test}} < \log k$ in expectation. If $\text{MDL} = n \log k + \epsilon$ for some $\epsilon \geq 0$, then:
    \begin{equation}
    \begin{split}
        \mathbb{E}[\text{EDL}] &= \mathbb{E}[\text{MDL}] - n \cdot \mathbb{E}[L_{\text{test}}] \\
        &\leq n \log k + \epsilon - n \log k = \epsilon.
    \end{split}
    \end{equation}

    For well-behaved models, $\epsilon \to 0$, giving $\mathbb{E}[\text{EDL}] \to 0$.
\end{proof}

\begin{figure}
    \centering
        \includegraphics[width=1.0\columnwidth]{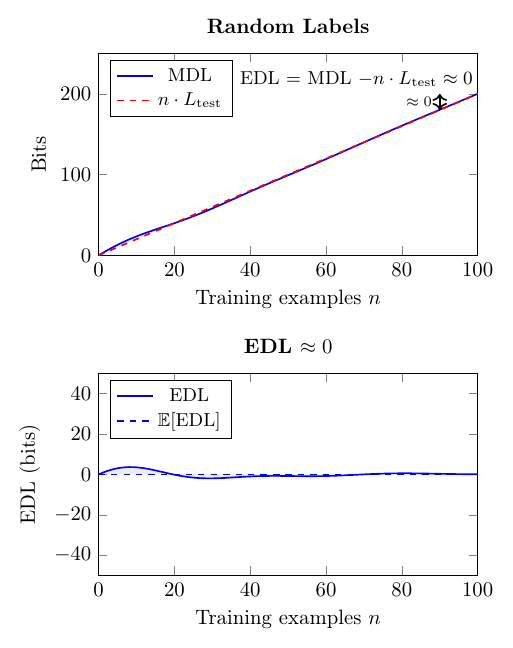}
    \caption{When labels are random, test loss $L_\text{test}$ remains constant since no generalizable pattern can be learned. MDL and the residual codelength $n \cdot L_\text{test}$ both increase linearly (in expectation), proportional to the number of training examples $n$. Since no learnable structure exists in the data, $\text{EDL} = \text{MDL} - n \cdot L_\text{test}\approx 0$, and the amount of generalizable information absorbed is negligible.}
    \label{fig:RandomLabelsDiagramAppendix}
\end{figure}

\subsection{Hypothesis Collapse (\autoref{sec::HypothesisCollapse})}
\label{Sec::HypothesisCollapseAppendix}

\textit{Setup.} A learner considers $m$ hypotheses $\mathcal{H} = \{h_1, \ldots, h_m\}$ with uniform prior $P(h_j) = 1/m$. Each hypothesis $h_j: \mathcal{X} \to \mathcal{Y}$ determines the label for each input. A \textit{diagnostic} example $(x, y)$ is one for which only a single hypothesis is consistent:
\begin{equation}
    |\{h \in \mathcal{H} : h(x) = y\}| = 1.
\end{equation}

\begin{proposition}[Reproduced from \autoref{prop:hypothesiscollapse}]
\label{prop:hypothesiscollapse}
    Consider a Bayesian learner with hypothesis class $\mathcal{H} = \{h_1, \ldots, h_m\}$ and uniform prior. Suppose:
    \begin{enumerate}[label=\roman*]
        \item The example $(x, y)$ is diagnostic, meaning that exactly one hypothesis $h_j$ satisfies $h_j(x)=y$.
        \item The hypothesis are maximally distinguishing on $x$: for each label $y' \in \mathcal{Y}$, exactly $m/k$ hypotheses predict $y'$ for input x (where $k = |\mathcal{Y}|$).
    \end{enumerate}
    Then a single example contributes exactly $\log k$ bits to EDL, and the generalization improvement is $\log m$ bits, the entropy of the full hypothesis space.

    More generally, without assumption (ii), a single example can contribute at most $\mathrm{min}\left(\log k, \log m\right)$ bits to EDL, where the codelength contribution is determined by the label entropy under the predictive distribution, while the generalization improvement is determined by the hypothesis entropy reduction.
\end{proposition}

Restating to highlight the case of maximum contribution to EDL from a single example:
\begin{proposition}
    A single diagnostic example can contribute up to (but not necessarily) $\log m$ bits to EDL.
\end{proposition}

\begin{figure}
    \centering
        \includegraphics[width=1.0\columnwidth]{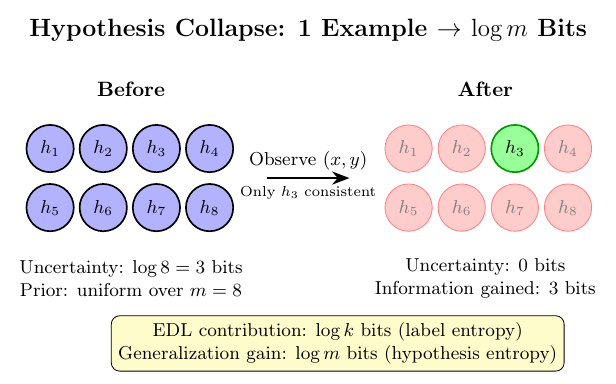}
    \caption{A single example can contribute up to $\log m$ bits of generalizable information by reducing uncertainty about which hypothesis is correct.}
    \label{fig:HypothesisCollapseDiagramAppendix}
\end{figure}

\begin{proof}
    Before observing any examples, the learner's predictive distribution for input $x$ is:
    \begin{equation}
    \begin{split}
        p_0(y|x) &= \sum_{j=1}^m P(h_j) \cdot \mathbf{1}[h_j(x) = y] \\
        &= \frac{|\{h \in \mathcal{H} : h(x) = y\}|}{m}.
    \end{split}
    \end{equation}

    For a diagnostic example $(x_1, y_1)$, suppose $h_1$ is the unique consistent hypothesis. If the hypotheses predict uniformly distributed labels for $x_1$ (\textit{i.e.}, each label in $\mathcal{Y}$ is predicted by $m/k$ hypotheses where $k = |\mathcal{Y}|$), then:
    \begin{equation}
        p_0(y_1|x_1) = \frac{1}{k}.
    \end{equation}

    The loss on this example is:
    \begin{equation}
    \begin{split}
        \ell(\theta_0; x_1, y_1) &= -\log p_0(y_1|x_1) = \log k.
    \end{split}
    \end{equation}

    After observing $(x_1, y_1)$, the learner updates to the posterior:
    \begin{equation}
    \begin{split}
        P(h_j | x_1, y_1) &= \begin{cases}
            1 & \text{if } h_j(x_1) = y_1 \\ 0 & \text{otherwise}
        \end{cases} = \begin{cases}
            1 & \text{if } j = 1 \\ 0 & \text{otherwise}.
        \end{cases} 
    \end{split}
    \end{equation}

    The learner now knows $h_1$ is correct. For any future input $x$:
    \begin{equation}
        p_1(y|x) = \mathbf{1}[h_1(x) = y].
    \end{equation}

    The test loss becomes $L_{\text{test}}(\theta^*) = 0$ if $h_1$ is deterministic (\textit{i.e.}, each input has a unique correct label under $h_1$).

    The EDL from this single example is:
    \begin{equation}
    \begin{split}
        \text{EDL} &= \ell(\theta_0; x_1, y_1) - 1 \cdot L_{\text{test}}(\theta^*) \\
        &= \log k - 0 \\
        &= \log k.
    \end{split}
    \end{equation}

    Note that $\log k$ may equal $\log m$ if $k = m$. In this case, the information gained about which hypothesis is correct is $\log m$, the entropy of the uniform distribution over $m$ hypotheses reduced to zero.

    It is important to clarify that the EDL contribution from this example is $\log k$ (the codelength), not $\log m$ (the hypothesis entropy associated with specifying a single hypothesis among a set of $m$ hypotheses), though these coincide when $k = m$, \textit{i.e.}, when each hypothesis predicts a distinct label for the diagnostic input. More generally:
    \begin{itemize}
        \item Codelength contribution to MDL: \\$-\log p_0(y_1 \mid x_1) = \log k$ (under assumption (ii), where predictions are uniform over labels
        \item Hypothesis entropy reduction: $\log m$ (from a uniform prior over $m$ possible hypotheses, $P(h_j)=1/m$, to certainty on one hypothesis, $P(h_1)=1$)
        \item Generalization improvement: $L(\theta_0) - L(\theta^*)$, which depends on on how hypotheses differ on future inputs
    \end{itemize}

    The relevant insight for elicitation is that, even when the codelength contribution is small (\textit{e.g.}, $\log k = 1$ bit for binary labels), the generalization improvement can be large ($\log m$ bits) if a single diagnostic example disambiguates among many latent capabilities. And indeed, this is precisely the elicitation setting: minimal information in the training signal unlocks substantial predictive improvement.

    To make this precise, consider the \textit{generalization} enabled by identifying $h_1$. Before the example, predicting on a new input $x'$ requires averaging over hypotheses:
    \begin{equation}
        p_0(y \mid x') = \frac{1}{m} \sum_j \mathbf{1}[h_j(x') = y].
    \end{equation}

    If hypotheses disagree substantially on $x'$, this average is spread out, yielding high entropy. After identifying $h_1$, predictions are deterministic.

    The expected reduction in loss per example is:
    \begin{equation}
    \begin{split}
        L(\theta_0) - L(\theta^*) &= \mathbb{E}_x\left[\log \frac{1}{p_0(Y \mid x)}\right] - 0 \\
        &= \mathbb{E}_x[H(Y \mid x; \theta_0)],
    \end{split}
    \end{equation}
    where $H(Y|x; \theta_0)$ is the entropy of the initial predictive distribution for input $x$.

    In the extreme case where all $m$ hypotheses predict distinct labels for a typical input, such that $k = m$ and each label has probability $1/m$:
    \begin{equation}
        L(\theta_0) - L(\theta^*) = \log m.
    \end{equation}

    Thus, EDL from one example is:
    \begin{equation}
        \text{EDL} = n \cdot (L(\theta_0) - L(\theta^*)) = 1 \cdot \log m = \log m.
    \end{equation}
\end{proof}

\begin{remark}
    This result explains how elicitation can succeed with minimal data. If a pretrained model has $m$ ``latent hypotheses'' about how to respond to a task, a single diagnostic example can select the correct one, providing $\log m$ bits of task-relevant information regardless of the label's own entropy.
\end{remark}

\subsection{Disjoint Subdistributions (\autoref{sec::DisjointSubdistributions})}
\label{Sec::DisjointSubdistributionsAppendix}

\textit{Setup.} The data distribution is a mixture of $K$ disjoint subdistributions:
\begin{equation}
    \mathcal{D} = \sum_{j=1}^K \pi_j \mathcal{D}_j,
\end{equation}
where $\pi_j > 0$, $\sum_j \pi_j = 1$, and each $\mathcal{D}_j$ has support $\mathcal{X}_j$ with $\mathcal{X}_j \cap \mathcal{X}_{j'} = \emptyset$ for $j \neq j'$. Within each subdistribution $\mathcal{D}_j$, there is a learnable rule that, once acquired, reduces per-example loss by $\Delta_j$ bits.

\begin{proposition}[Reproduced from \autoref{prop:SubdistributionWeightEDL}]
    Learning a rule that reduces loss by $\Delta$ bits on a subdistribution with mixture weight $\pi$ contributes approximately $\pi \cdot \Delta$ bits per example to expected generalization improvement.
\end{proposition}

\begin{proof}
    Let $L_j^0$ denote the initial expected loss on subdistribution $\mathcal{D}_j$, and $L_j^*$ the loss after learning the rule (so $L_j^0 - L_j^* = \Delta_j$).

    The initial expected loss under the full distribution $\mathcal{D}$ is:
    \begin{equation}
        L(\theta_0) = \sum_{j=1}^K \pi_j L_j^0.
    \end{equation}
    
    Suppose training provides sufficient examples from subdistribution $j$ to learn its rule, but no examples from other subdistributions. The final model achieves:
    \begin{itemize}
        \item Loss $L_j^* = L_j^0 - \Delta_j$ on $\mathcal{D}_j$
        \item Loss $L_{j'}^0$ on $\mathcal{D}_{j'}$ for $j' \neq j$ (no improvement without training data)
    \end{itemize}
    
    The final expected loss under $\mathcal{D}$ is:
    \begin{equation}
        L(\theta^*) = \pi_j L_j^* + \sum_{j' \neq j} \pi_{j'} L_{j'}^0 = \pi_j (L_j^0 - \Delta_j) + \sum_{j' \neq j} \pi_{j'} L_{j'}^0.
    \end{equation}
    
    The improvement is:
    \begin{equation}
    \begin{split}
        L(\theta_0) - L(\theta^*) &= \left[ \pi_j L_j^0 + \sum_{j' \neq j} \pi_{j'} L_{j'}^0 \right] \\
        &\phantom{=} - \left[ \pi_j (L_j^0 - \Delta_j) - \sum_{j' \neq j} \pi_{j'} L_{j'}^0 \right] \\
        &= \pi_j \Delta_j.
    \end{split}
    \end{equation}
    
    For a single subdistribution with weight $\pi$ and improvement $\Delta$:
    \begin{equation}
        L(\theta_0) - L(\theta^*) = \pi \cdot \Delta.
    \end{equation}
\end{proof}

\textit{EDL Calculation.} Now consider training on $n$ examples drawn from $\mathcal{D}$ (so each subdistribution $j$ receives approximately $n \pi_j$ examples in expectation).

For a subdistribution $j$ with $n_j \approx n \pi_j$ examples:
\begin{itemize}
    \item MDL contribution: The first-epoch loss on these examples transitions from $L_j^0$ (early) to $L_j^*$ (after learning). Approximately:
        \begin{equation}
            \text{MDL}_j \approx n_j \cdot \bar{L}_j,
        \end{equation}
        where $\bar{L}_j$ is the average loss during learning, satisfying $L_j^* \leq \bar{L}_j \leq L_j^0$. \\
    \item Test loss contribution: $\pi_j \cdot L_j^*$.
\end{itemize}

Summing over subdistributions,
\begin{equation}
    \text{MDL} = \sum_j \text{MDL}_j \approx \sum_j n_j \bar{L}_j = n \sum_j \pi_j \bar{L}_j.
\end{equation}
\begin{equation}
    n \cdot L_{\text{test}}(\theta^*) = n \sum_j \pi_j L_j^*.
\end{equation}

Therefore,
\begin{equation}
    \text{EDL} = n \sum_j \pi_j (\bar{L}_j - L_j^*) = n \sum_j \pi_j \cdot \text{(learning gap for } j\text{)}.
\end{equation}

The per-example EDL is:
\begin{equation}
    \frac{\text{EDL}}{n} = \sum_j \pi_j (\bar{L}_j - L_j^*).
\end{equation}

This is the $\pi$-weighted average of the learning gaps across subdistributions, confirming that each subdistribution's contribution to EDL is proportional to its mixture weight.

\subsection{Dynamics of the Coupon Collector's Problem (\autoref{sec::CouponCollector})}
\label{Sec::CouponCollectorAppendix}

\textit{Setup.} The task has $K$ distinct ``concepts.'' Each training example is drawn uniformly from one of the $K$ concepts. The model must observe at least one example of each concept to generalize fully. Let $C(n)$ denote the number of distinct concepts observed after $n$ examples.

\begin{proposition}[Reproduced from \autoref{prop:couponcollector}]
    For a task requiring coverage of $K$ concepts, EDL shows an accelerating phase as $n$ approaches $K \ln K$, followed by saturation.
\end{proposition}

\begin{proof}
    We model the learning dynamics as follows:
    \begin{itemize}
        \item A concept is ``known'' after observing at least one example of it.
        \item For a known concept, the model achieves loss $L_{\text{low}}$ (near zero).
        \item For an unknown concept, the model achieves loss $L_{\text{high}}$.
        \item Let $\Delta = L_{\text{high}} - L_{\text{low}}$ be the per-concept improvement.
    \end{itemize}

    \textit{Coverage dynamics.} The number of distinct concepts encountered after $n$ draws follows the coupon collector's problem distribution:
    \begin{equation}
        \mathbb{E}[C(n)] = K \left(1 - \left(1 - \frac{1}{K}\right)^n\right) \approx K(1 - e^{-n/K}).
    \end{equation}

    Full coverage requires $\mathbb{E}[N] = K H_K \approx K \ln K$ examples, where $H_K = \sum_{i=1}^K 1/i$ is the $K$-th harmonic number.

    \textit{MDL computation.} At step $i$, the example comes from a new concept with probability $(K - C(i-1))/K$. Recall that the loss is:
    \begin{equation}
        \ell_i = \begin{cases} L_{\text{high}} & \text{if concept is new} \\ L_{\text{low}} & \text{if concept is known}. \end{cases}
    \end{equation}

    The expected MDL is:
    \begin{equation}
    \begin{split}
        \mathbb{E}[\text{MDL}] &= \sum_{i=1}^n \mathbb{E}[\ell_i] \\
        &= \sum_{i=1}^n \left[\frac{K - \mathbb{E}[C(i-1)]}{K} \cdot L_{\text{high}} + \frac{\mathbb{E}[C(i-1)]}{K} \cdot L_{\text{low}}\right].
    \end{split}
    \end{equation}

    Using $\mathbb{E}[C(i-1)] \approx K(1 - e^{-(i-1)/K})$:
    \begin{equation}
        \mathbb{E}[\text{MDL}] \approx \sum_{i=1}^n \left[e^{-(i-1)/K} \cdot L_{\text{high}} + (1 - e^{-(i-1)/K}) \cdot L_{\text{low}}\right].
    \end{equation}

    Converting to an integral for large $K$:
    \begin{equation}
    \begin{split}
        \mathbb{E}[\text{MDL}] &\approx \int_0^n \left[e^{-t/K} L_{\text{high}} + (1 - e^{-t/K}) L_{\text{low}}\right] dt \\
        &= n L_{\text{low}} + K(1 - e^{-n/K}) \Delta.
    \end{split}
    \end{equation}

    As the number of examples $n \to \infty$: 
    \begin{equation}
        \mathbb{E}[\text{MDL}] \to n L_{\text{low}} + K \Delta.
    \end{equation}

    \textit{Test loss.} The test loss depends on the fraction of concepts known:
    \begin{equation}
    \begin{split}
        L_{\text{test}} &= \frac{K - \mathbb{E}[C(n)]}{K} \cdot L_{\text{high}} + \frac{\mathbb{E}[C(n)]}{K} \cdot L_{\text{low}} \\
        &= L_{\text{low}} + \frac{K - \mathbb{E}[C(n)]}{K} \cdot \Delta.
    \end{split}
    \end{equation}

    Using $\mathbb{E}[C(n)] \approx K(1 - e^{-n/K})$:
    \begin{equation}
        L_{\text{test}} \approx L_{\text{low}} + e^{-n/K} \cdot \Delta.
    \end{equation}

    \textit{EDL calculation.}

    \begin{equation}
        \mathbb{E}[\text{EDL}] = \mathbb{E}[\text{MDL}] - n \cdot L_{\text{test}}.
    \end{equation}

    Substituting:
    \begin{equation}
    \begin{split}
        \mathbb{E}[\text{EDL}] &\approx \left[n L_{\text{low}} + K(1 - e^{-n/K}) \Delta\right] - n \left[L_{\text{low}} + e^{-n/K} \Delta\right]\\
        &= K(1 - e^{-n/K}) \Delta - n e^{-n/K} \Delta \\
        &= \Delta \left[K(1 - e^{-n/K}) - n e^{-n/K}\right].
    \end{split}
    \end{equation}

    Let $u = n/K$. Then:
    \begin{equation}
    \begin{split}
        \frac{\mathbb{E}[\text{EDL}]}{K \Delta} &= (1 - e^{-u}) - u e^{-u} \\
        &= 1 - (1 + u)e^{-u}.
    \end{split}
    \end{equation}

    \textit{Scaling regimes.}

    \textit{1. Small $n$ (i.e., $u \ll 1$):}
    Taylor expansion gives $(1+u)e^{-u} \approx 1 - u^2/2$, so:
    \begin{equation}
        \frac{\mathbb{E}[\text{EDL}]}{K \Delta} \approx \frac{u^2}{2} = \frac{n^2}{2K^2}.
    \end{equation}
    Thus, $\mathbb{E}[\text{EDL}] \approx \frac{\Delta n^2}{2K}$, which is \textit{quadratic} in $n$.

    \textit{2. Large $n$ (i.e., $u \gg 1$):}
    $(1+u)e^{-u} \to 0$, so:
    \begin{equation}
        \mathbb{E}[\text{EDL}] \to K \Delta.
    \end{equation}
    EDL saturates at the total information across all concepts.

    \textit{Per-example EDL.}

    \begin{equation}
        \frac{\mathbb{E}[\text{EDL}]}{n} = \frac{\Delta}{u} \left[1 - (1+u)e^{-u}\right].
    \end{equation}

    For small $u$: $\mathbb{E}[\text{EDL}]/n \approx u \Delta / 2 = n \Delta / (2K)$, which \textit{increases} with $n$.

    For large $u$: $\mathbb{E}[\text{EDL}]/n \approx K\Delta / n$, which \textit{decreases} with $n$.

    Solving $(d/du)[u^{-1}(1-(1+u)e^{-u})] = 0$ numerically gives $u \approx 1.79$, so $n \approx 1.79 K$. Thus, The maximum of $\mathbb{E}[\text{EDL}]/n$ occurs at intermediate $n$, around $n \approx K$.

    This yields a characteristic inverted ``U'' shape for coupon collector learning dynamics (in terms of generalizable information absorbed from $n$ train data examples): EDL per example ($\text{EDL}/n$) increases during the learning phase (as coverage builds), peaks near complete coverage, then decreases as additional examples become redundant.
\end{proof}

\begin{figure}
    \centering
        \includegraphics[width=1.0\columnwidth]{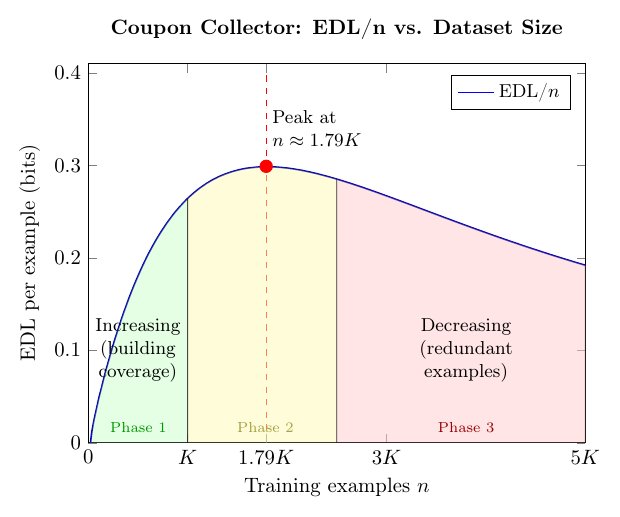}
    \caption{In the low-coverage regime (Phase 1, green shaded region), new examples increase $\mathbb{E}[\text{EDL}]/n$, providing more opportunities to learn generalizable patterns from relatively few known concepts. The rate of information absorption peaks as full coverage of concepts is approached (Phase 2, yellow shaded region). Beyond the coverage threshold (peak), learning enters the full-coverage regime (Phase 3, red shaded region), where $\mathbb{E}[\text{EDL}]/n$ begins to decrease as additional examples yield diminishing amounts of predictive information that can be absorbed.}
    \label{fig:CouponCollectorDiagramAppendix}
\end{figure}

\textit{Summary of coupon collector scaling dynamics:}

As a function of $n$ training examples encountered and $K$ total concepts (coupons) to be collected, the per-example EDL scales as:
\begin{equation}
        \frac{\mathbb{E}[\text{EDL}]}{n} = \frac{K\Delta}{n} \left[1 - \left( 1 + \frac{n}{K} \right)e^{-n/K}\right].
    \end{equation}

\begin{itemize}
    \item For small $n$ (few examples):
    \begin{equation}
        \frac{\mathbb{E}[\text{EDL}]}{n} \approx \frac{n\Delta}{2K}
    \end{equation}
    \item For small $n$ (many examples):
    \begin{equation}
        \frac{\mathbb{E}[\text{EDL}]}{n} \approx \frac{K\Delta}{n}
    \end{equation}
\end{itemize}

The coverage transition occurs around $n \approx 1.79K$, where $\mathbb{E}[\text{EDL}]/n$ is maximized:
\begin{equation}
\begin{split}
    \mathrm{argmax}_n\left(\frac{\mathbb{E}[\text{EDL}]}{n}\right) &\approx 1.79K \\
    \left(\frac{\mathbb{E}[\text{EDL}]}{n}\right)_{n\approx1.79K} &\approx 0.298\Delta
\end{split}
\end{equation}
In the low-coverage regime below this threshold, new examples increase predictive information absorbed, dominating the learning dynamics even when partial coverage has been achieved. Beyond this threshold, learning enters the full-coverage regime, where $\mathbb{E}[\text{EDL}]/n$ begins to decrease as additional examples yield diminishing amounts of information absorbed.

\subsection{Format Learning Transients (\autoref{sec::FormatLearning})}
\label{Sec::FormatLearningAppendix}

\textit{Setup.} A task has two components:
\begin{itemize}
    \item Format ($F$): Output structure (e.g., ``The answer is X''). Low entropy, learned quickly.
    \item Capability ($C$): The actual predictive relationship associated with the capability of interest. High entropy, learned slowly.
\end{itemize}

\begin{proposition}[Paraphrased from \autoref{prop:formatlearning}]
    When a task requires both format and capability learning, EDL/example may show non-monotonic behavior with dataset size.
\end{proposition}

\begin{proof}
    We analyze three regimes.

    \textit{Regime 1: $n < n_F$ (Neither format nor capability learned)}

    During training, both components contribute to loss. Assuming linear learning for simplicity:
    \begin{equation}
        L(\theta_i) \approx L_F^0 (1 - i/n_F) + L_C^0 (1 - i/n_C) \quad \text{for } i < n_F.
    \end{equation}

    Summing losses during the first epoch to obtain MDL:
    \begin{equation}
    \begin{split}
        \text{MDL} &\approx \sum_{i=0}^{n-1} \left[L_F^0 (1 - i/n_F) + L_C^0 (1 - i/n_C)\right] \\
        &= n L_F^0 - \frac{L_F^0 n(n-1)}{2n_F} + n L_C^0 - \frac{L_C^0 n(n-1)}{2n_C}.
    \end{split}
    \end{equation}

    For $n \ll n_F, n_C$:
    \begin{equation}
        \text{MDL} \approx n(L_F^0 + L_C^0) = n L_{\text{total}}^0.
    \end{equation}

    The test loss after $n$ examples is:
    \begin{equation}
        L_{\text{test}} \approx L_F^0(1 - n/n_F) + L_C^0(1 - n/n_C).
    \end{equation}

    The EDL is:
    \begin{equation}
    \begin{split}
        \text{EDL} &\approx n L_{\text{total}}^0 - n \left[L_F^0\left(1 - \frac{n}{n_F}\right) + L_C^0\left(1 - \frac{n}{n_C}\right)\right] \\
        &= \frac{n^2 L_F^0}{n_F} + \frac{n^2 L_C^0}{n_C} \\
        &= n^2 \left( \frac{L_F^0}{n_F} + \frac{L_C^0}{n_C} \right)
    \end{split}
    \end{equation}

    Dividing by $n$ yields the per-example EDL:
    \begin{equation}
        \frac{\text{EDL}}{n} \approx n \left( \frac{L_F^0}{n_F} + \frac{L_C^0}{n_C} \right),
    \end{equation}
    which is an increasing function of $n$. Thus, per-example EDL \textit{increases} with $n$ as both components are being learned.

    \textit{Regime 2: $n_F < n < n_C$ (Format learned, capability learning in progress)}

    Format is fully learned; only capability contributes to ongoing learning.

    Test loss is given by:
    \begin{equation}
        L_{\text{test}} \approx L_C^0 \left(1 - \frac{n}{n_C} \right).
    \end{equation}

    MDL splits into two parts:
    \begin{itemize}
        \item First $n_F$ examples: average loss $\approx L_F^0/2 + L_C^0$, contributing $\approx n_F(L_F^0/2 + L_C^0)$.
        \item Next $(n - n_F)$ examples: format already learned, average loss during capability learning $\approx L_C^0(1 - (n-n_F)/(2n_C))$ for the middle of this interval.
    \end{itemize}

    For simplicity, approximate:
    \begin{equation}
        \text{MDL} \approx n_F \cdot \frac{L_F^0 + 2L_C^0}{2} + (n - n_F) \cdot L_C^0 \cdot \left(1 - \frac{n - n_F}{2n_C}\right).
    \end{equation}

    As $n$ increases through this regime, MDL grows roughly linearly (each new example costs $\approx L_C^0(1 - \text{progress})$), while test loss decreases (capability improves).

    The EDL is:
    \begin{equation}
        \text{EDL} = \text{MDL} - n \cdot L_{\text{test}}.
    \end{equation}

    The behavior of the per-example EDL depends on the relative rates. Typically:
    \begin{itemize}
        \item If capability learning is slow ($n_C \gg n$), test loss decreases slowly, and $\text{EDL}/n$ may remain roughly constant or slightly increase.
        \item As $n \to n_C$, test loss drops significantly, and EDL accumulates the full capability information.
    \end{itemize}

    \textit{Regime 3: $n > n_C$ (Both format and capability learned)}

    Both components are learned fully. Test loss $L_{\text{test}} \approx 0$.

    Accordingly, the MDL saturates, as all information has been provided during the first $n_C$ examples. (Recall that subsequent examples do not provide information after all patterns/algorithms have been fully learned.)
    \begin{equation}
        \text{MDL} \approx n_F \cdot \frac{L_F^0}{2} + n_C \cdot \frac{L_C^0}{2},
    \end{equation}
    where the factor of 1/2 approximates the average loss during learning each component.

    Subtracting the test loss from the MDL:
    \begin{equation}
        \text{EDL} \approx \frac{n_F L_F^0 + n_C L_C^0}{2} - n \cdot 0 = \frac{n_F L_F^0 + n_C L_C^0}{2},
    \end{equation}
    which is constant for $n > n_C$.

    The per-example EDL is:
    \begin{equation}
        \frac{\text{EDL}}{n} \approx \frac{n_F L_F^0 + n_C L_C^0}{2n},
    \end{equation}
    which \textit{decreases} as $1/n$.

    Summary of $\text{EDL}/n$ trajectory:
    \begin{enumerate}
        \item $n < n_F$: Increasing (both components learning).
        \item $n_F < n < n_C$: May plateau or slowly vary (format learning complete, active capability learning ongoing).
        \item $n > n_C$: Decreasing as $1/n$ (both saturated, redundant examples).
    \end{enumerate}

    The specific shape depends on the relative magnitudes of $L_F^0$, $L_C^0$, $n_F$, and $n_C$. If $L_F^0 \ll L_C^0$ and $n_F \ll n_C$, the format transient is a brief initial period, and the main EDL accumulation occurs during capability learning.
\end{proof}

\begin{figure}
    \centering
        \includegraphics[width=1.0\columnwidth]{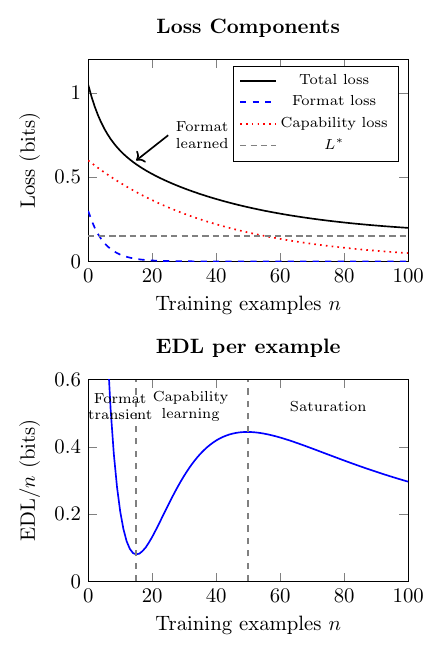}
    \caption{\textbf{Format learning can induce transients in EDL scaling and dynamical behavior.} In teaching scenarios, format learning can cause a transient drop in the per-example EDL ($\text{EDL}/n$), as format comprises generalizable task structure that can be easily learned with low sample complexity. This exploitation of superficial regularity in the data can dominate the early scaling dynamics of $\text{EDL}/n$ when the number of training examples $n$ is small and significant task-relevant capability has not yet been acquired. Before format-invariant task features have been learned, additional examples beyond the minimum required to align model outputs to the task specifications do not marginally improve predictions. Once sufficient examples have been trained on, capability begins to develop, and $\text{EDL}/n$ increases until task performance or parameter capacity are saturated. Once saturation is reached, the model absorbs generalizable information at a decreasing rate, either because additional examples provide little predictive information (capability saturation) or because the model cannot absorb additional information about remaining structure that exists in the data (capacity saturation).}
    \label{fig:FormatLearningDiagramAppendix}
\end{figure}

\textit{Practical Implications.} When computing EDL on tasks with both format and capability components, scoring only the capability-relevant tokens (e.g., answer tokens, not formatting tokens) can help isolate the capability learning signal and avoid artifacts from the format transient. Alternatively, teaching the task format prior to or separately from the task of interest (e.g., few-shot prompting, providing clear instructions/examples, instruction tuning) can mitigate significant contributions to EDL from format learning.

\section{Additional Technical Results}
\label{Sec::AdditionalResultsAppendix}

\subsection{Variance of EDL}
\label{Sec::EDLVarianceAppendix}

The EDL computed from a single dataset is a random variable. Understanding its variance is important for interpreting measurements.

\begin{proposition}
    Under mild regularity conditions, the variance of EDL scales as:
    \begin{equation}
        \text{Var}[\mathrm{EDL}] = O(n).
    \end{equation}
\end{proposition}

\begin{proof}[Proof sketch.]
    The MDL is a sum of $n$ loss terms, each with bounded variance. By standard results on sums of random variables (accounting for weak dependence through the evolving parameters), $\text{Var}[\text{MDL}] = O(n)$.

    The test loss estimator has variance $O(1/n_{\text{test}})$ where $n_{\text{test}}$ is the test set size, so $\text{Var}[n \cdot L_{\text{test}}] = O(n^2 / n_{\text{test}})$.

    For $n_{\text{test}} = \Omega(n)$, the dominant contribution is from MDL, giving $\text{Var}[\text{EDL}] = O(n)$.

    The standard deviation is $O(\sqrt{n})$, so relative uncertainty in $\text{EDL}/n$ is $O(1/\sqrt{n})$, vanishing for large $n$.
\end{proof}

\subsection{Effect of Example Ordering}
\label{Sec::ExampleOrderingAppendix}

\begin{proposition}
    For i.i.d. data, the expected MDL is independent of example ordering.
\end{proposition}

\begin{proof}
    Let $\sigma$ be any permutation of $\{1, \ldots, n\}$. Under the permuted ordering, MDL becomes:
    \begin{equation}
        \text{MDL}_\sigma = \sum_{i=1}^n \ell(\theta_{i-1}^\sigma; x_{\sigma(i)}, y_{\sigma(i)}),
    \end{equation}
    where $\theta_i^\sigma$ denotes the parameters after training on $(x_{\sigma(1)}, y_{\sigma(1)}), \ldots, (x_{\sigma(i)}, y_{\sigma(i)})$.
    
    For i.i.d. data, the joint distribution of $(x_{\sigma(1)}, y_{\sigma(1)}), \ldots, (x_{\sigma(n)}, y_{\sigma(n)})$ is the same as that of $(x_1, y_1), \ldots, (x_n, y_n)$.
    
    Therefore, the distribution of $\text{MDL}_\sigma$ is the same as that of $\text{MDL}$, and in particular:
    \begin{equation}
        \mathbb{E}[\text{MDL}_\sigma] = \mathbb{E}[\text{MDL}].
    \end{equation}
\end{proof}

\subsection{Continuous-Time Approximation}
\label{Sec::ContinuousApproximationAppendix}

For theoretical analysis, it is sometimes convenient to approximate the discrete training trajectory with a continuous flow.

\begin{definition}
    The continuous-time MDL is
    \begin{equation}
        \text{MDL}_{\text{cont}}(T) = \int_0^T L(\theta(t)) \, dt,
    \end{equation}
where $\theta(t)$ follows a gradient flow $\dot{\theta} = -\nabla L(\theta)$ and $T$ corresponds to seeing $n$ examples.
\end{definition}

\begin{proposition}
    Under appropriate scaling, $\text{MDL}_{\text{cont}}(T)$ approximates $\text{MDL}$ as the learning rate $\eta \to 0$ and $n \to \infty$ with $n\eta = T$ fixed.
\end{proposition}

This continuous approximation enables analysis using ODE techniques and provides insights into the geometry of learning trajectories.

\section{Additional Information-Theoretic Details and Explanations}
\label{sec::InformationTheoryAppendix}

\subsection{Learning as Compression: A Detailed Picture}

\paragraph{Notation.}
Let \(D = \{(x_i,y_i)\}\) be a dataset sampled i.i.d. from distribution $\mathcal{D}$ over $\mathcal{X} \times \mathcal{Y}$, where $x_i$ are the inputs to the model (\textit{i.e.}, examples from the dataset) and $y_i$ are the corresponding data labels for some language modeling task. Let us imagine that Alice (who has all \((x_i,y_i)\) pairs in $D$) wants to communicate a model $M^*$ (fine-tuned from a base model $M_0$ of model class $\mathcal{M}$ on all data pairs \((x_i,y_i) \in D\) to a final performance level $L^*$) to Bob, who only has the inputs $x_i$ in $D$. 

\paragraph{Setting.}
Instead of sending Bob the model weights themselves, which may be very difficult for large models, Alice can alternatively send the labels of the dataset she used to train her model $M^*$, such that Bob may train an identical model $M^*$ with the same performance $L^*$ using the same base model $M$ that she used. Alice's task is then to optimally encode the full set of labels $y_{i=1,...,n}$ such that she uses the minimum number of bits necessary to transmit them to Bob. This codelength, the minimum description length, corresponds to the minimum amount of information needed for Bob to train a copy of the same (initial) base model $M_0$ to an equivalent performance level.

To make Bob's fine-tuning process more efficient, Alice and Bob both agree to use identical copies of $M_0$ (the base model Alice used for training $M^*$) as the algorithm used to compress, encode, and decode the data labels. They agree on a choice of learning algorithm $A$ to use, including any variables (such as hyperparameters, seeds, and optimizer) that affect the training dynamics, which enables them to train identical models given the same training data. 

\paragraph{Prequential MDL and online coding.}
Using the same model as Bob, Alice encodes and sends the labels one (batch) at a time, with the cross-entropy between the current model's prediction and the true label for each example (batch) determining the minimum number of bits Alice must use to transmit it. After encoding each label (batch of labels) using the model, Alice sends it to Bob, who then uses his identical copy of the current model to decode the correct label (batch). Once Bob has decoded the label(s), he and Alice then each train their respective copy of the model on the label(s) Bob received according to $A$, such that their models remain identical at each step.

After training on each new label (batch), their models successively improve at predicting subsequent labels, such that each new label Alice encodes has a smaller cross-entropy (log-loss) with the prediction from the current, slightly better model. Accordingly, the number of bits required for Alice to transmit the next label decreases as she and Bob simultaneously train (their own exact copies of) a model that better predicts the data. After all labels have been sent, the total (minimal) number of bits needed to transmit them is the minimum description length, given by the sum of the cross-entropy losses for all of the individual labels in the dataset, or equivalently, the area under the training curve for the first epoch. 

Once Bob receives the full set of labels, he can continue to train his copy of the model for additional epochs on those same labels to further improve its performance (following the directions of Alice to train an identical copy of $M^*$). Because Alice no longer needs to transmit anything additional for Bob to continue training his model (we assume the set of instructions for how to train $M^*$ is known in advance and included in $A$), the total information required for her to communicate her fine-tuned model is merely the number of bits needed to transmit the necessary labels for Bob to be able to train an identical copy of it.

\paragraph{Excess description length and bits absorbed during training.}
Once Bob trains his model on all data to convergence (as judged by performance on some validation metric), its cross-entropy loss when evaluated on the test set\footnote{We assume that the test set is sampled i.i.d. from and accurately reflects performance on the true data distribution the model is trained to predict.} reflects the average number of bits this best, final model $M^*$ requires to encode labels for examples drawn from the same distribution. This is the remaining information in the data that that the model cannot compress further by training for additional epochs on the current train set (i.e., the minimum bits still required to communicate the correct labels even knowing the correct model parameters) reflecting the residual codelength of the data once the model parameters are known.

Because the minimum description length represents the smallest amount of information necessary to represent both the model \textit{and} the data, the difference between the cross-entropy of the converged, best model $M^*$ trained on all data and the online codelength (the cumulative train set loss during the model's first pass over the data/area under the first-epoch train set loss curve) corresponds to the cost of the model. This ``excess description length'' beyond what is necessary to reproduce the correct labels using the converged model is the information in the data that the model \textit{does} fully compress during training, reducing its loss from its zero-shot value $L_0$ to $L^*$ as a result of fine-tuning. Alternatively stated, the excess description length is the total \textit{generalizable} information in the data that is absorbed by the model during training---the information in the data which contributes to better performance on the test set and which the model's parameters alone are sufficient to reproduce/represent---given by the difference between the prequential MDL and the remaining cross-entropy of the best model trained on that data. 

Minimum description length quantifies how much information is supplied to a model, as measured by the model itself, by employing the model's intrinsic ability to generate its own most compact, optimal description of that information. MDL grows linearly with the number of examples in the dataset, with asymptotic behavior determined by the irreducible error in the dataset which can't be compressed; as more information is supplied to the model, the model's description length of the data correspondingly increases to account for the additional information, with a lower bound on the minimum additional information per example set by the irreducible error. In practice, this is approximated by the test loss of the best trained model---a loss ``floor'' that the model converges to in the asymptotic data limit, which describes the bits of the fine-tuning dataset (i.e., general information about the task) that the model cannot absorb. It is the best achievable loss on the true distribution given the training data (for the specific hypothesis class).

However, the amount of information that a model can store in its parameters is bounded, with a strict (trivial) upper limit on the number of bits that can be stored in any single parameter enforced by its numerical precision. While MDL can exceed the capacity of the trainable parameters, as it merely quantifies the amount of information required to transmit the labels, EDL describes the amount of generalizable information absorbed by the parameters, which provides a tighter upper bound on the amount of predictive information that can be extracted from the train data during learning by the model. For suitable choices of learning algorithms, the compression achieved through learning can be near-optimal, such that EDL can provide substantially better bounds on new information obtained and stored during fine-tuning.

\subsection{Additional benefits to EDL}

The Minimum Description Length principle \citep{rissanen1978modeling} states that the best model is the one that provides the shortest description of both itself and the data it is used to model. Though MDL has the advantage that it can be used to determine these most parsimonious descriptions, it does not report on a model's ultimate performance following training. 

Because prequential minimum description length is computed through online training, it is susceptible to early training dynamics, such as slow initial learning from a low learning rate, that result in an overestimation of the minimum codelength. It can be difficult to distinguish from the MDL alone whether its value reflects a model's genuine ability to efficiently compress the data or simply a suboptimal choice of hyperparameters. 

While validation accuracy (or some other final performance metric) can describe a model's final performance, it does not describe the amount of information required to achieve it. Complicated models that overfit the data perform well on validation accuracy/loss, but their complexity indicates that they themselves are difficult to model, meaning that the information necessary to model the data has merely shifted to the model parameters.

Surplus description length determines the amount of information necessary to train an $\varepsilon-$loss predictor, but it requires samples to be drawn i.i.d. from an (effectively) infinite distribution such that the model converges below the $\varepsilon-$loss threshold at a sample complexity far below the total number of examples in the distribution. This means that it cannot be applied to many finite data settings. Though it is less sensitive to hyperparameters and early training dynamics than prequential MDL, for models which converge quickly to the loss threshold or which fail to learn substantially at the beginning of training, early training dynamics can significantly impact the measured SDL.

Excess description length measures the generalizable information in the train dataset that a model extracts during training, corresponding to learned structure or regularity in the data that predicts the overall distribution those data were sampled from. Using the test loss as proxy for the model's generalization error over the distribution, EDL is computable for any dataset of any size. In this way, EDL also provides an implicit measure of sample complexity: for a given test performance, lower EDL indicates fewer unique examples were required.

In realistic training settings, it is common to train for multiple epochs on the same dataset, as generalization often improves with subsequent passes over the data. EDL describes how well a model can extract learnable structure from repeated exposure to the same data: increases in EDL in later epochs reflect additional predictive information extracted from the train data.

An additional advantage of EDL is that it is compute-bound-agnostic: EDL quantifies the information cost of generalization for any training procedure under any computational bound or stopping condition. While EDL can be computed at any point(s) in training, it is especially powerful for estimating the maximum amount of generalizable information a model can extract from a finite dataset given unbounded compute. For a finite dataset, maximum EDL can be obtained by training to best achievable performance on the validation set (with early stopping prior to overfitting) and evaluating test loss on the best model checkpoint.

This use of EDL is particularly relevant for many practical settings where dataset size is often limited, such as fine-tuning, reinforcement learning, or other post-training methods. Maximum EDL estimates the maximum amount of learnable structure that a model can extract from the train data, providing insight into how far a finite amount of data can be stretched. Taken in consideration with the learning regime, this can be used to estimate the minimum information required to achieve a desired test performance. This is important for safety cases: given equal compute and final model performance, there is a meaningful difference between needing to train on a small vs. a large dataset, as lower information barriers indicate more easily elicited capabilities.

\begin{figure*}
    \centering
        \includegraphics[width=1.0\textwidth]{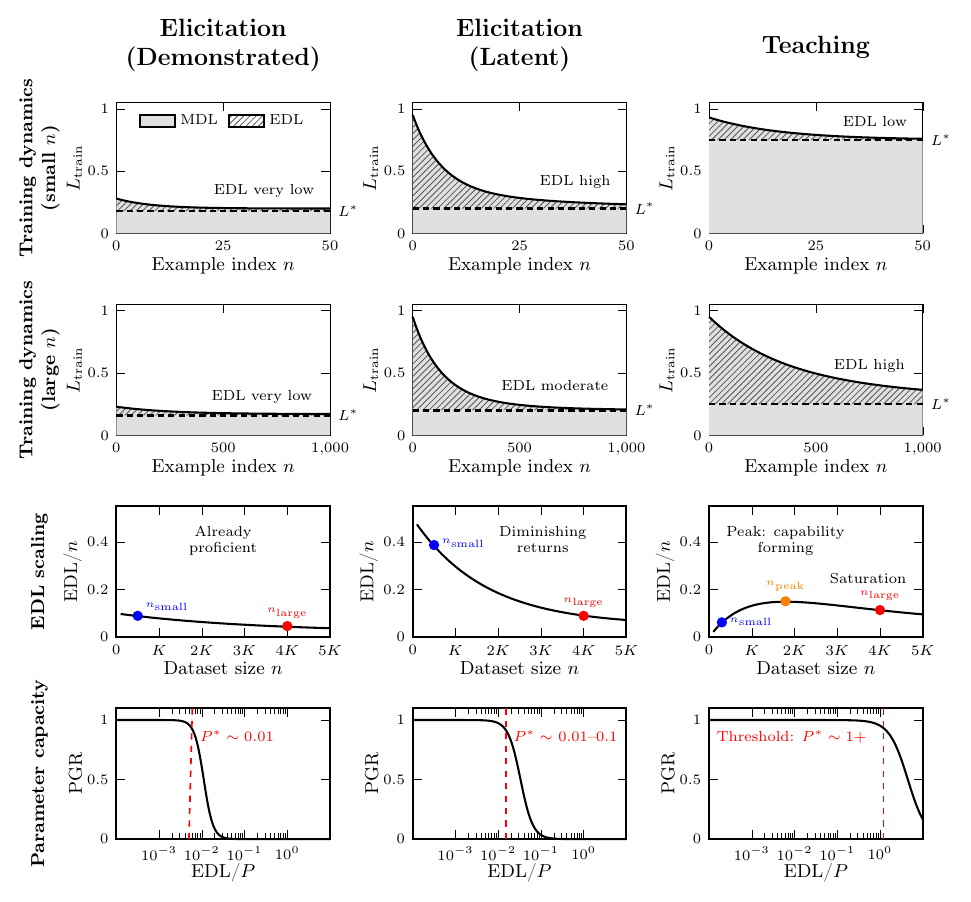}
    \caption{Conceptual illustration of eliciting latent capabilities (left column), teaching new capabilities (center column), and eliciting demonstrated capabilities (right column) in small and large data regimes. (Top row: small $n$, second row: large $n$) The minimum description length (MDL) is the total area under the training curve $L_\text{train}$ (all gray shaded areas), representing the number of bits provided to the model (through the labels of the fine-tuning dataset), as quantified by the model itself. The final converged test loss (after training for as many epochs as necessary) $L^*$ represents the model's asymptotic error on the true data distribution when trained on $n$ unique examples. The excess description length (EDL) is the area between the training curve $L_\text{train}$ and the generalization error $L^*$, which represents the generalizable information that is absorbed by the model's parameters during training, corresponding to the reduction in its prediction error on the underlying data distribution. (Third row) Scaling behavior of EDL per example ($\text{EDL}/n$ as dataset size $n$ increases, where dataset size is normalized by the total number of constituent concepts that must be known for task mastery $K$. Models that already know the necessary task components but improve their ability to employ that knowledge (elicitation) absorb less generalizable information with additional examples, whereas models that must learn the requisite concepts from scratch experience increasing returns when developing fundamental task capability. (Bottom row) The fraction of performance gap recovered (PGR) versus EDL per parameter reveals different parameter capacity limits for elicitation and teaching: models which require less information to achieve maximal task performance (elicitation) saturate parameter capacity at lower information thresholds than models which require more information to learn a task.}
    \label{fig:EDLDiagramAppendix}
\end{figure*}